\documentclass{article}

\usepackage[utf8]{inputenc}
\usepackage[english]{babel}
\usepackage[T1]{fontenc}

\usepackage{amsthm}
\usepackage{mathrsfs}
\usepackage{fancyhdr}
\usepackage[colorlinks,linkcolor=red,citecolor=blue,urlcolor=black,filecolor=blue,backref=page]{hyperref}
\usepackage{graphicx}
\usepackage{amsmath}
\usepackage{amssymb}
\usepackage{dsfont}

\usepackage{algorithm}
\usepackage{algpseudocode}

\usepackage{natbib}

\usepackage{authblk}
\usepackage{fullpage}

\usepackage{subcaption} 

\theoremstyle{plain}
\newtheorem{theorem}{Theorem}[section]
\newtheorem{lemma}[theorem]{Lemma}

\newtheorem{proposition}[theorem]{Proposition}

\theoremstyle{remark}

\newcommand{\prob}{\mathbb{P}}
\newcommand{\mean}{\mathbb{E}}
\newcommand{\Var}{\mathbb{V}}

\newcommand{\cov}{\text{cov}}

\newcommand{\ud}{\mathop{}\!\mathrm{d}}
\newcommand{\cond}{\,|\,}
\newcommand{\mcond}{\,\middle|\,}

\newcommand{\e}{e} 

\DeclareMathOperator{\Normal}{\mathcal{N}}

\DeclareMathOperator{\GP}{\mathcal{GP}}

\newcommand{\Ba}{\mathbf{a}}

\newcommand{\Bx}{\mathbf{x}}

\newcommand{\BB}{\mathbf{B}}

\newcommand{\BS}{\mathbf{S}}

\newcommand{\BSigma}{\boldsymbol{\Sigma}}

\newcommand{\Btheta}{\boldsymbol{\theta}}
\newcommand{\Bvartheta}{\boldsymbol{\vartheta}}

\newcommand{\Bphi}{\boldsymbol{\phi}}

\renewcommand{\epsilon}{\varepsilon}

\newcommand{\reals}{\mathbb{R}}

\newcommand{\onevector}{\mathds{1}}

\newcommand{\Bzeros}{\mathbf{0}}
\newcommand{\Id}{\mathbf{I}}

\newcommand{\myquad}{\quad}
\newcommand{\idxt}{1:t}
\newcommand{\idxtp}{1:t+1}
\newcommand{\pacc}{p_{a}}
\newcommand{\Deltav}{\tau}

\newcommand{\maxvar}{maxvar}

\newcommand{\randmaxvar}{rand\_maxvar}

\newcommand{\expintvar}{expintvar}
\newcommand{\expdiffvar}{expdiffvar}

\newcommand{\bigO}{\mathcal{O}} 

\newcommand{\probmeas}{\Pi}

\newcommand{\piabcnoeps}{\pi^{\text{ABC}}}
\newcommand{\piabc}{\pi^{\text{ABC}}_{\varepsilon}}
\newcommand{\tildepiabc}{\tilde{\pi}^{\text{ABC}}_{\varepsilon}}
\newcommand{\hatpiabc}{\widehat{\pi}^{\text{ABC}}_{\varepsilon}}
\newcommand{\pitrue}{\pi^{\text{true}}}
\newcommand{\tildepiNabc}{\tilde{\pi}^{\text{ABC}}_{\text{N}}}

\usepackage{color} 
\newcommand{\revcol}{black}

\usepackage{booktabs} 

\def\app#1#2{%
  \mathrel{%
    \setbox0=\hbox{$#1\sim$}%
    \setbox2=\hbox{%
      \rlap{\hbox{$#1\propto$}}%
      \lower1.1\ht0\box0%
    }%
    \raise0.25\ht2\box2%
  }%
}
\def\approxprop{\mathpalette\app\relax}

\newcommand{\indic}{\mathds{1}}

\title{Efficient acquisition rules for model-based approximate Bayesian computation} 

\author[1]{Marko J\"{a}rvenp\"{a}\"{a}}
\author[2]{Michael U. Gutmann}
\author[1]{Arijus Pleska}
\author[1]{Aki Vehtari}
\author[1]{Pekka Marttinen}

\affil[1]{Helsinki Institute for Information Technology HIIT, Department of Computer Science, Aalto University}
\affil[2]{School of Informatics, University of Edinburgh}

\begin{document}

\maketitle

\begin{abstract}

Approximate Bayesian computation (ABC) is a method for Bayesian inference when the likelihood is unavailable but simulating from the model is possible. However, many ABC algorithms require a large number of simulations, which can be costly. To reduce the computational cost, Bayesian optimisation (BO) and surrogate models such as Gaussian processes have been proposed. 
Bayesian optimisation enables one to intelligently decide where to evaluate the model next 
{\color{\revcol} but common BO strategies are not designed for the goal of estimating the posterior distribution. }
Our paper addresses this gap in the literature. 
We propose to compute the uncertainty in the ABC posterior density, which is due to a lack of simulations to estimate this quantity accurately, and define a loss function that measures this uncertainty. We then propose to select the next evaluation location to minimise the expected loss. 
Experiments show that the proposed method often produces the most accurate approximations as compared to common BO strategies. 


\end{abstract}

\section{Introduction} \label{sec:intro}

We consider the problem of Bayesian inference of some unknown parameter $\Btheta \in \Theta \subset \reals^p$ of a simulation model. Such models are typically not amenable to any analytical treatment but they can be simulated with any parameter $\Btheta \in \Theta$ to produce data $\Bx_{\Btheta} \in \mathcal{X}$. 
Simulation models are also called simulator-based or implicit models \citep{Diggle1984}. 
Our prior knowledge about the unknown parameter $\Btheta$ is represented by the prior probability density $\pi(\Btheta)$ and the goal of the analysis is to update our knowledge about the parameters $\Btheta$ after we have observed data $\Bx_{obs}\in \mathcal{X}$.

If evaluating the likelihood function $\pi(\Bx\cond\Btheta)$ is feasible, the posterior distribution can be computed directly using Bayes' theorem 
\begin{align}
\pi(\Btheta \cond \Bx_{obs}) 
&= \frac{\pi(\Btheta)\pi(\Bx_{obs}\cond\Btheta)}{\int_{\Theta} \pi(\Btheta')\pi(\Bx_{obs}\cond\Btheta') \ud \Btheta'} 
\propto \pi(\Btheta)\pi(\Bx_{obs}\cond\Btheta). \label{eq:bayes}
\end{align}
{\color{\revcol} In this article we focus on simulation models that have intractable likelihoods. This means that} one can only simulate from the model, that is, draw samples $\Bx_{\Btheta} \sim \pi(\cdot\cond\Btheta)$, but not evaluate the likelihood function $\pi(\Bx_{obs}\cond\Btheta)$ {\color{\revcol} at all so that} the standard Bayesian approach cannot be used. 
%
{\color{\revcol} For example, possibly high-dimensional unobservable latent random quantities present in the simulation model can make evaluating the likelihood impossible. Such difficulties occur in many areas of science, and typical application fields include population genetics \citet{Beaumont2002,Numminen2013}, genomics \citet{Marttinen2015,Jarvenpaa2016}, ecology \citep{Wood2010,Hartig2011} and psychology \citet{Turner2012}, see e.g.~\citet{Lintusaari2016} and references therein for further examples. 
}

Approximate Bayesian computation (ABC) replaces likelihood evaluations with model simulations\footnote{{\color{\revcol}Such approaches are also called likelihood-free in the literature although this name can be considered a misnomer. Namely, while the user does not need to provide the likelihood of the simulator model, many methods construct some sort of likelihood approximation implicitly or explicitly.}}, see e.g.~\citet{Marin2012,Turner2012,Lintusaari2016} for an overview.  
The main idea of the basic ABC algorithm is to draw a parameter value from the prior distribution, simulate a data set with the given parameter value, and accept the value as a draw from the (approximate) posterior if the {discrepancy} between the simulated and observed data is small enough.
This algorithm produces samples from the approximate posterior distribution
\begin{align}
\piabc(\Btheta \cond \Bx_{obs}) \propto \pi(\Btheta)\int \pi_{\epsilon}(\Bx_{obs}\cond\Bx) \pi(\Bx\cond\Btheta) \ud \Bx, \label{eq:abc_post}
\end{align}
where $\pi_{\epsilon}(\Bx_{obs}\cond\Bx) \propto \indic_{\Delta(\Bx_{obs},\Bx) \leq \epsilon}$, although other choices of $\pi_{\epsilon}$ are also possible \citep{Wilkinson2013}. The function $\Delta: \mathcal{X} \times \mathcal{X} \rightarrow \reals_{+}$ is the discrepancy that tells how different the simulated and observed data sets are, and it is often formed by combining a set of summary statistics 
{\color{\revcol} even though, occasionally, the output data have a relatively small dimension and the data sets can be compared directly. Sometimes the discrepancy function may be available from previous analyses with similar models or can be constructed based on expert knowledge of the application field. 
The discrepancy and the summaries affect the approximations and their choice is an active research topic \citep{Blum2013,Fearnhead2012, Gutmann2017}. In this work, we are concerned with another, equally important, research question, namely given a suitable discrepancy function, how to perform the inference in a computationally efficient manner. 
}
%

The threshold $\epsilon$ controls the trade-off between the accuracy of the approximation and computational cost: a small $\epsilon$ yields accurate approximations but requires more simulations, see e.g.~\citet{Marin2012}.
Given $t$ samples from the model for some $\Btheta$, that is, $\Bx^{(i)}_{\Btheta} \sim \pi(\cdot\cond\Btheta)$ for $i=1,\ldots,t$, the value of the ABC posterior in Equation (\ref{eq:abc_post}) can be estimated as
\begin{align}
\piabc(\Btheta \cond \Bx_{obs}) \approxprop \pi(\Btheta) \sum_{i = 1}^{t}  \pi_{\epsilon}(\Bx_{obs}\cond\Bx^{(i)}_{\Btheta}), \label{eq:abc_post_approx}
\end{align}
%
where ``$\approxprop$'' means that the left-hand side is approximately proportional to the right-hand side {\color{\revcol} and where the extra approximation is due to replacing the integral with the Monte-Carlo sum}.


Algorithms based on Markov Chain Monte Carlo and sequential Monte Carlo have been used to improve the efficiency of ABC 
as compared to the basic rejection sampler \citep{Marjoram2003,Sisson2007,Beaumont2009,Toni2009,Marin2012,Lenormand2013}. 
Unfortunately, the sampling based methods still require a very large number of simulations. In this paper we focus on the challenging scenario where the number of available simulations is limited, e.g.~to fewer than a thousand, rendering these sampling-based ABC methods infeasible. 
Different modelling approaches have also been proposed to reduce the number of simulations required. For example, in the synthetic likelihood method summary statistics are assumed to follow the Gaussian density \citep{Wood2010,Price2016} and the resulting likelihood approximation can be used together with MCMC {\color{\revcol} but evaluating the synthetic likelihood is typically still very expensive. }
\citet{Wilkinson2014,Meeds2014,Jabot2014,Kandasamy2015,Drovandi2015b,Gutmann2015,Jarvenpaa2016} all use Gaussian processes (GP) to accelerate ABC in various ways. Some other alternative approaches are considered by \citet{Fan2013,Papamakarios2016}. 
%
Also, \citet{Beaumont2002,Blum2010,Blum2010b} have used modelling as a post-processing step to correct the approximation error of the nonzero threshold.

While probabilistic modelling has been used to accelerate ABC inference, and strategies have been proposed for selecting which parameter to simulate next, little work has focused on trying to quantify the amount of uncertainty in the estimator of the ABC posterior density under the chosen modelling assumptions. This uncertainty is due to a finite computational budget to perform the inference and could be thus also called as ``computational uncertainty''. 
Consequently, little has been done to design strategies that directly aim to minimise this uncertainty. 
%
To our knowledge, only \citet{Kandasamy2015} have used the uncertainty in the likelihood function to propose new simulation locations in a query-efficient way, but they assumed that the likelihood can be evaluated, although with high a computational cost. Also, \citet{Wilkinson2014} modelled the uncertainty in the likelihood to rule out regions with negligible posterior probability. \citet{Rasmussen2003} used GP regression to accelerate Hybrid Monte Carlo but did not consider the setting of ABC. \citet{Osborne2012} developed an active learning scheme to select evaluations to estimate integrals such as the model evidence under GP modelling assumptions, however, their approach is designed for estimating this particular scalar value. 
%
Finally, \citet{Gutmann2015} proposed Bayesian optimisation to efficiently select new evaluation locations. {\color{\revcol} While the BO strategies they used to illustrate the framework worked reasonably, their approach does not directly address the goal of ABC, that is to learn the posterior accurately. } 

In this article we propose an acquisition function for selecting the next evaluation location tailored specifically for ABC. The acquisition function measures the expected uncertainty in the estimate of the (unnormalised) ABC posterior density function over a future evaluation of the simulation model, and proposes the next simulation location so that this expected uncertainty is minimised. We also consider some variants of this strategy. 
%
More scecifically, in Section \ref{sec:problem} we formulate our probabilistic approach on a general level. In Section \ref{sec:gp_models} we propose a particular algorithm, based on modelling the discrepancy with a GP. Section \ref{sec:experiments} contains experiments. Some additional details of our algorithms are discussed in Section \ref{sec:discussion} and Section \ref{sec:conclusions} concludes the article. Technical details and additional experiments are presented in the supplementary material. 

\section{Problem formulation} \label{sec:problem}


We start by presenting the main idea of the probabilistic framework for query-efficient ABC inference. Suppose we have training data $D_{\idxt} = \{(\Bx_i,\Btheta_i)\}_{i=1}^t$ of simulation outputs $\Bx_i\in\mathcal{X}$ that were generated by simulating the model with parameters $\Btheta_i$. Suppose also that we have a Bayesian model that describes our uncertainty about the future simulation output $\Bx^{*} \in \mathcal{X}$ with parameter $\Btheta^{*}$ conditional on the training data $D_{\idxt}$. This uncertainty is represented by a probability measure $\probmeas(\ud \Bx^{*} \cond \Btheta^{*}, D_{\idxt})$\footnote{{\color{\revcol} We use $\probmeas(\cdot)$ to denote the probability measure of a random quantity that can be interpreted from the argument. Similarly, $\pi(\cdot)$ denotes a probability density function whenever the corresponding random vector is assumed to be absolutely continuous.}}. 
Instead of modelling the full data $\Bx^{*}\in\mathcal{X}$, we note that in practice it is reasonable to model only some summary statistics $s(\Bx^{*})\in S \subset \reals^r$. Alternatively, the discrepancy between the observed data and simulator output can be modelled as is done later in this article.
Importantly, our estimate for the ABC posterior probability density function $\piabcnoeps$ actually depends on the training data if e.g.~Equation (\ref{eq:abc_post_approx}) is used, and can therefore also be considered a random quantity. Given the training data $D_{\idxt}$, we assume that, using our Bayesian model, we can represent the uncertainty in $\piabcnoeps$ using a probability measure $\probmeas(\ud \piabcnoeps \cond D_{\idxt})$ over the space of (suitable smooth) density functions $\piabcnoeps : \Theta \rightarrow \reals_{+}$, {\color{\revcol} where the probability measure $\probmeas$ now describes the uncertainty in the ABC posterior.} 

If the amount of available simulations is limited due to a high computational cost, we may have considerable uncertainty of the ABC posterior $\piabcnoeps$. Let $\mathcal{L}_{\piabcnoeps}(D_{1:t})$ denote the loss due to our uncertainty about the ABC posterior density. This loss function could, for example, measure overall uncertainty in the probability density $\piabcnoeps$ or the uncertainty of a particular point estimate of interest such as the posterior mean. {\color{\revcol} In the latter case, for a scalar $\theta$, we could choose $\mathcal{L}_{\piabcnoeps}(D_{1:t}) = \Var(\int_{\Theta} \theta \piabcnoeps(\theta) \ud \theta \cond D_{1:t})$, where the variance (assuming it exists) is taken with respect to the probability measure $\probmeas(\ud \piabcnoeps \cond D_{\idxt})$. }
%

We consider the sequential setting where, at each iteration, we need to decide the next evaluation location. After each iteration, we can compute the uncertainty in the ABC posterior and the corresponding loss function, and fit a model that predicts the next simulation output, given all data available at the time. 
{\color{\revcol} Our aim is to choose the next evaluation location $\Btheta^{*} = \Btheta_{t+1}$ such that the expected loss, after having simulated the model at this location, is minimised.} 
That is, we want to minimise
\begin{align}
&\mean_{\Bx^{*} \cond \Btheta^{*}, D_{\idxt}}(\mathcal{L}_{\piabcnoeps}(D_{\idxt} \cup \{\Bx^{*}, \Btheta^{*}\})) 
%
= \int \mathcal{L}_{\piabcnoeps}(D_{\idxt} \cup \{\Bx^{*}, \Btheta^{*}\}) \probmeas(\ud \Bx^{*} \cond \Btheta^{*}, D_{\idxt}) \label{eq:loss}
\end{align}
%
with respect to $\Btheta^{*}$, where we need to average over unknown unknown simulator {\color{\revcol} output $\Bx^{*} = \Bx_{t+1}$ at parameter $\Btheta^{*}$ using the model for the new simulator output $\probmeas(\ud \Bx^{*} \cond \Btheta^{*}, D_{\idxt})$.}
If the loss function measures the uncertainty of the ABC posterior density, then this approach is{\color{\revcol}, by construction, a one step-ahead} Bayes optimal solution to a decision problem of minimising the expected uncertainty 
and offers thus a query-efficient approach to determine the {\color{\revcol} next evaluation location.}

This approach resembles the entropy search (ES) method \citep{Hennig2012,HernandezLobato2014}. Other related approaches have been proposed by \citet{Wang2016,Bijl2016,Wang2017}. Different from these approaches, our main goal is to select the parameter for a future run of the costly simulation model so that the uncertainty in the approximate posterior is minimised. ES, in contrast, is designed for query-efficient global optimisation and it aims to find a parameter value that maximises the objective function, and to minimise the uncertainty related to this maximiser. 
%
We note that the rationale of our approach is essentially the same as in probabilistic numerics literature (see e.g.~\citet{Hennig2015}) or in sequential Bayesian experimental design (see \citet{Ryan2016} for a recent survey). 
However, different from these approaches, our interest is to design the evaluations to minimise the uncertainty in a quantity that itself describes the uncertainty of the parameters of a costly simulation model. 
The uncertainty in the former is due to a limited budget for model simulations that we can control, while the uncertainty in the latter is caused by noisy observations that have already been provided to us and are considered here as fixed.

The framework outlined above requires some modelling choices and can lead to computational challenges as the selection of the future evaluation location itself can require costly evaluations (as is the case of ES). 
In the following section we propose an efficient algorithm based on a loss function that measures the uncertainty in the (unnormalised) ABC posterior over the parameter space and a GP surrogate model. 
We also consider some alternative strategies that are more heuristic but easier to evaluate. 
{\color{\revcol} While our approach can be extended to a batch setting where multiple acquisitions are computed in parallel, in this article we restrict our discussion to the sequential case.}
We note that the outlined strategy is ``myopic'', meaning that the expected uncertainty after the {\color{\revcol} next evaluation} is considered only, and the number of simulations left in a limited budget is not taken into account, see e.g.~\citet{gonzalez2015glasses} for some discussion in BO context; non-myopic strategies are {\color{\revcol}also} beyond the scope of this work.

Details of our approach appear in the next section, but the main idea is illustrated in Figure \ref{fig:demo}. We model the discrepancy $\Delta_{\Btheta} = \Delta(\Bx_{obs},\Bx_{\Btheta})$ with GP regression (Figure~\ref{fig:demo}a). The ABC posterior is proportional to the prior times the probability of obtaining a discrepancy realisation that is below the threshold when the model is simulated. However, because the GP is fitted with limited training data, this probability cannot be estimated exactly, causing uncertainty in the ABC posterior density function (Figure \ref{fig:demo}b). We propose an acquisition function that selects the next evaluation location to minimise the expected variance of the (unnormalised) ABC posterior density over the parameter space. 


\begin{figure}[ht]
\centering
\includegraphics[width=0.5\textwidth]{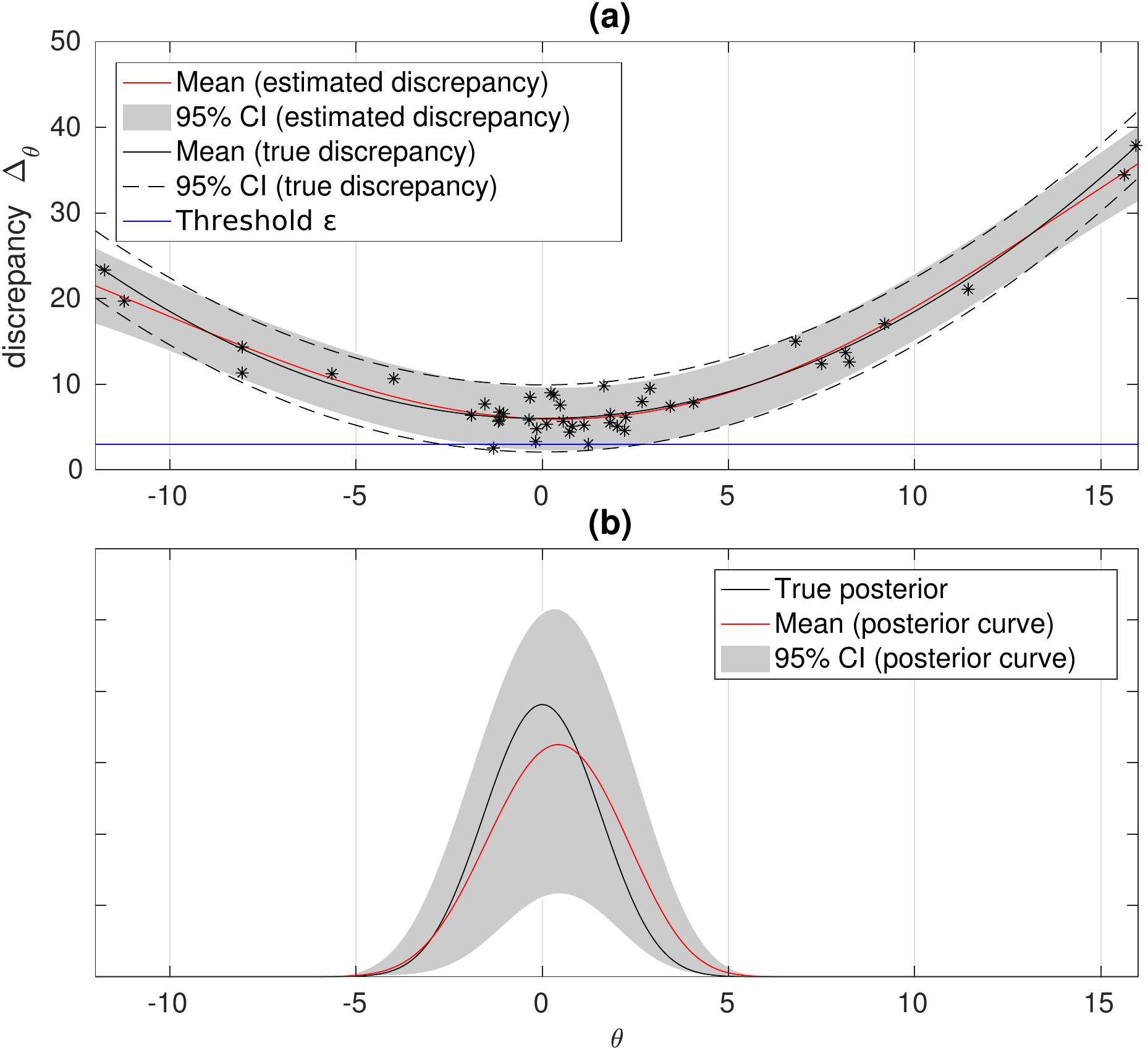}
\caption
{(a) The estimated and true discrepancy distributions are compared. Most of the evaluations are successfully chosen on the modal area of the posterior, leading to a good approximation there. 
(b) The red curve shows the mean of the unnormalised posterior density function and the grey area its 95\% pointwise credible interval.
} \label{fig:demo}
\end{figure}

\section{Nonparametric modelling and parameter acquisition} \label{sec:gp_models}





This section contains the details of our algorithms. Section \ref{subsec:gp_model} describes the GP model for the discrepancy, which permits closed-form equations for many of the required quantities to estimate the posterior, which are derived in Section \ref{subsec:uncertainty}. In Sections \ref{subsec:acquisition_rule} and \ref{subsec:alternatives} we formulate the  proposed acquisition functions, and handling uncertainty in GP hyperparameters is briefly discussed in Section \ref{subsec:gp_hypers}.

\subsection{GP model for the discrepancy}\label{subsec:gp_model}


We consider the discrepancy $\Delta_{\Btheta}$ a stochastic process indexed by $\Btheta$ i.e.~a random function of the parameter $\Btheta$. 
We assume that the discrepancy can be modelled by a Gaussian distribution\footnote{Alternatively, we could model some transformation of the discrepancy such as $\log \Delta_{\Btheta}$. In that case, the the following analysis goes similarly.}, that is~$\Delta_{\Btheta} \sim \Normal(f(\Btheta),\sigma_n^2)$ for some unknown suitably smooth function $f:\Theta\rightarrow\reals$ and variance $\sigma_n^2\in\reals_{+}$ both of which need to be estimated. We place a Gaussian process prior on $f$ so that $f \sim \GP(\mu(\Btheta),k(\Btheta,\Btheta'))$. While other choices are also possible, in this paper we consider $\mu(\Btheta) = 0$, and use the squared exponential covariance function $k(\Btheta,\Btheta') = \sigma_f^2 \exp(-\sum_{i=1}^p(\theta_i-\theta_i')^2/(2l_i^2))$. There are thus $p+2$ hyperparameters to infer, denoted by $\Bphi = (\sigma_f^2,l_1,\ldots,l_p,\sigma_n^2)$. 

Conditioned on the obtained training data $D_{\idxt} = \{(\Delta_i,\Btheta_i)\}_{i=1}^t$, which consists of realised discrepancy-parameter pairs, and the GP hyperparameters $\Bphi$, our knowledge of the function $f$ evaluated at an arbitrary point $\Btheta\in\Theta$ can be shown to be $f(\Btheta) \cond D_{\idxt}, \Btheta, \Bphi \sim \Normal(m_{\idxt}(\Btheta),v_{\idxt}^2(\Btheta))$, where 
\begin{align}
m_{\idxt}(\Btheta) &= k(\Btheta,\Btheta_{\idxt})K(\Btheta_{\idxt})^{-1}\Delta_{\idxt}, \label{eq:gp_mean}
\\
v_{\idxt}^2(\Btheta) &= k(\Btheta,\Btheta) - k(\Btheta,\Btheta_{\idxt}) K^{-1}(\Btheta_{\idxt}) k(\Btheta_{\idxt},\Btheta) \label{eq:gp_var}
\end{align}
and $K(\Btheta_{\idxt}) = k(\Btheta_{\idxt},\Btheta_{\idxt}) + \sigma_n^2\Id$. Above we defined $k(\Btheta,\Btheta_{\idxt}) = (k(\Btheta,\Btheta_1),\ldots,k(\Btheta,\Btheta_t))^T$, $k(\Btheta_{\idxt},\Btheta_{\idxt})_{ij} = k(\Btheta_i,\Btheta_j)$ for $i,j=1,\ldots,t$ and similarly for $k(\Btheta_{\idxt},\Btheta)$. We have also used $\Delta_{\idxt} = (\Delta_1,\ldots,\Delta_t)^T$. 
A comprehensive presentation of GP regression can be found in \citet{Rasmussen2006}.

\subsection{Quantifying the uncertainty of the ABC posterior estimate} \label{subsec:uncertainty}

As in \citet{Gutmann2015}, one can compute the posterior predictive density for a new discrepancy value at $\Btheta$ using $\Delta_{\Btheta} \cond D_{\idxt}, \Btheta, \Bphi \sim \Normal(m_{\idxt}(\Btheta), v_{\idxt}^2(\Btheta) + \sigma_n^2)$ and obtain a model-based estimate for the acceptance probability given the modelling assumptions and training data $D_{\idxt}$, 
as 
\begin{align}
\prob(\Delta_{\Btheta} \leq \epsilon \cond D_{\idxt}, \Btheta, \Bphi) = \Phi\left({(\epsilon - m_{\idxt}(\Btheta))}/{\sqrt{\sigma_n^2 + v_{\idxt}^2(\Btheta)}}\right),
\label{eq:p_prob}
\end{align}
where $\Phi(z) = \int_{-\infty}^z \exp(-t^2/2)\ud t / (2\pi)$ is the cdf of the standard normal distribution. 
This probability is approximately proportional to the likelihood and yields a useful point estimate of the likelihood function. We here take a different approach and explicitly exploit the fact that part of the probability mass of $\Delta_{\Btheta}$ is due to our uncertainty in the latent function $f$ and GP hyperparameters $\Bphi$. For simplicity, we first assume that the GP hyperparameters $\Bphi$ are known and discuss relaxing this assumption in a later section. Indeed, if we knew $f$, the (unnormalised) ABC posterior $\tildepiabc(\Btheta)$ and the acceptance probability $\pacc(\Btheta)$\footnote{This notation should not be confused with a probability distribution function which is always denoted with $\pi(\cdot)$. 
}
could be computed as
\begin{align}
%
\tildepiabc(\Btheta) = \pi(\Btheta) \pacc(\Btheta), \quad 
%
\pacc(\Btheta) 
= \Phi\left(\frac{\epsilon - {f}(\Btheta)}{\sigma_n}\right). 
\label{eq:correct_estim_post}
\end{align}
With a limited number of discrepancy--parameter pairs in $D_{\idxt}$ there is uncertainty in the values of the function $f$ (and in GP hyperparameters $\Bphi$) which we propose to quantify and attempt to minimise in order to accurately estimate the ABC posterior. The following result (whose proof is found in the supplementary material) allows us to compute the expectation and the variance of the unnormalised ABC posterior.
%
%
\begin{lemma} \label{lemma:p_mean_var}
Under the GP model described in Section \ref{subsec:gp_model}, the pointwise expectation and variance of $\tildepiabc$ with respect to $\probmeas(\ud f \cond D_{\idxt})$ are 
\begin{align}
%
\mean(\tildepiabc(\Btheta) \cond D_{1:t}) 
&= \pi(\Btheta) \, \Phi\left( \frac{\epsilon - m_{\idxt}(\Btheta)}{\sqrt{\sigma_n^2 + v_{\idxt}^2(\Btheta)}} \right), \label{eq:post_mean} \\
%
\begin{split}
\Var(\tildepiabc(\Btheta) \cond D_{1:t}) 
&= \pi^2(\Btheta) \Bigg[ \Phi\left(\frac{\epsilon - m_{\idxt}(\Btheta)}{\sqrt{\sigma_n^2 + v_{\idxt}^2(\Btheta)}}\right)\Phi\left(\frac{m_{\idxt}(\Btheta) - \epsilon}{\sqrt{\sigma_n^2 + v_{\idxt}^2(\Btheta)}}\right) 
\\ 
%
&\myquad- 2 T\left( \frac{\epsilon - m_{\idxt}(\Btheta)}{\sqrt{\sigma_n^2 + v_{\idxt}^2(\Btheta)}},  \frac{\sigma_n}{\sqrt{\sigma_n^2 + 2v_{\idxt}^2(\Btheta)}} \right) \Bigg],
%
\label{eq:post_var}
\end{split}
\end{align}
%
where $m_{\idxt}(\Btheta)$ and $v_{\idxt}^2(\Btheta)$ are given by Equations (\ref{eq:gp_mean}) and (\ref{eq:gp_var}), respectively, and $T(\cdot,\cdot)$ is Owen's t-function which satisfies 
\begin{equation}
T(h,a) = \frac{1}{2\pi}\int_{0}^{a}\frac{\e^{-h^2(1+x^2)/2}}{1+x^2} \ud x,
\label{eq:owens_t}
\end{equation}
for $h,a\in \reals$. 
\end{lemma}
%

We note that Equation (\ref{eq:post_mean}) equals the product of the prior density $\pi(\Btheta)$ and the point estimate of the likelihood shown in Equation \ref{eq:p_prob} which was used in \citet{Gutmann2015}. The variance of the unnormalised ABC posterior in Equation (\ref{eq:post_var}) depends on Owen's t-function that needs to be computed numerically. However, there exists an efficient algorithm to evaluate its values by \citet{Patefield2000}.

It is of interest to examine when the variance in Equation (\ref{eq:post_var}) is large. If a parameter $\Btheta$ satisfies $m_{\idxt}(\Btheta) = \epsilon$, then the first term of Equation (\ref{eq:post_var}) is maximised, and in this case the second term is maximised for $\Btheta$ values where the posterior variance $v_{\idxt}^2(\Btheta)$ is large. On the other hand, if $m_{\idxt}(\Btheta) \gg \epsilon$ but $v_{\idxt}(\Btheta) \gg |m_{\idxt}(\Btheta) - \epsilon|$, the first term in Equation (\ref{eq:post_var}) is approximately maximised and the second term is also close to its maximum value, especially if also $v_{\idxt}(\Btheta) \gg \sigma_n$. Because the ABC threshold $\epsilon$ is usually chosen very small, we thus conclude that the variance in Equation (\ref{eq:post_var}) tends to be high in regions where the mean of the discrepancy $m_{\idxt}(\Btheta)$ is small and/or the variance of the latent function $v_{\idxt}^2(\Btheta)$ is large relative to the mean function.

Some further insight to Equation (\ref{eq:post_var}) is obtained by using the approximation $\Var_{f(\Btheta) \cond D_{\idxt}}(\tildepiabc(\Btheta)) \approx ((\tildepiabc)'(\mean_{f(\Btheta) \cond D_{\idxt}}(f(\Btheta))))^2 \Var_{f(\Btheta) \cond D_{\idxt}}(f(\Btheta))$, 
where the formula $(\tildepiabc)'(\mean_{f(\Btheta) \cond D_{\idxt}}(f(\Btheta)))$ denotes the derivative of $\tildepiabc$ with respect to $f(\Btheta)$ evaluated at $\mean_{f(\Btheta) \cond D_{\idxt}}(f(\Btheta))$. This approximation, also known as the delta method, produces
\begin{equation}
-\log \Var_{f(\Btheta) \cond D_{\idxt}}(\tildepiabc(\Btheta)) \approx \frac{(\epsilon - m_{\idxt}(\Btheta))^2}{\sigma_n^2} -\log (v_{\idxt}^2(\Btheta)) - 2\log\pi(\Btheta) + \log(2\pi\sigma_n^2), 
\label{eq:log_delta_approx}
\end{equation}
where the last term is constant and can be dropped. This equation has some similarity with the lower confidence bound (LCB) criteria used in bandit problems and Bayesian optimisation 
\begin{equation}
\text{LCB}(\Btheta) = m_{\idxt}(\Btheta) - \beta_t \sqrt{v_{\idxt}^2(\Btheta)},
\end{equation}
where $\beta_t$ is a tradeoff parameter. Both equations produce small values if the mean of the discrepancy $m_{\idxt}(\Btheta)$ is small (assuming also that $\epsilon \leq m_{\idxt}(\Btheta)$) or the variance $v_{\idxt}^2(\Btheta)$ is large relative to the mean $m_{\idxt}(\Btheta)$. However, the LCB tradeoff parameter $\beta_t$ typically depends on the iteration $t$ and the dimension of the parameter space (see \citet{Srinivas2010} for theoretical analysis) while in the posterior variance (Equation (\ref{eq:log_delta_approx})) this tradeoff is determined automatically in a nonlinear fashion and, unlike for LCB, the variance formula depends also on the prior density $\pi(\Btheta)$. 

Other useful facts for $\tildepiabc$ can also be obtained. Some of these formulas are not required to apply our methodology and are thus included as supplementary material. For instance, if $\pi(\Btheta)>0$ then the cdf for $\tildepiabc(\Btheta)$ is 
\begin{align}
F_{\tildepiabc(\Btheta)}(z) = \Phi\left(\frac{\sigma_n \Phi^{-1}(z/\pi(\Btheta)) + m_{\idxt}(\Btheta) - \epsilon}{v_{\idxt}(\Btheta)}\right), \label{eq:p_cdf}
\end{align}
if $z \in (0,\pi(\Btheta))$, zero if $z \leq 0$, and $1$ if $z \geq \pi(\Btheta)$. This formula enables the computation of quantiles which can be used for assessing the uncertainty via credible intervals. Setting $\alpha = F_{\tildepiabc(\Btheta)}(z)$, where $\alpha\in(0,1)$ and solving for $z$ yields the $\alpha$-quantile that was already used in Figure \ref{fig:demo}b,
\begin{align}
z_{\alpha} = \pi(\Btheta) \Phi\left(\frac{v_{\idxt}(\Btheta) \Phi^{-1}(\alpha) - m_{\idxt}(\Btheta) + \epsilon}{\sigma_n}\right). \label{eq:quantile}
\end{align}
From the above equation we see, e.g., that the median is given by $\pi(\Btheta) \Phi({(\epsilon - m_{\idxt}(\Btheta))}{/\sigma_n})$. 

Above we assumed that the GP hyperparameters $\Bphi$ are known but in practice these need to be estimated. One can use 
the MAP-estimate in the place of the fixed values in the previous formulae. The MAP-estimate is computed by maximising the logarithm of the marginal posterior
\begin{align}
\Bphi_{\idxt}^{\text{MAP}}
&= \arg\max_{\Bphi} \Big( \log\pi(\Bphi) - \frac{1}{2}\Delta_{\idxt}^T  K^{-1}(\Btheta_{\idxt})\Delta_{\idxt} - \frac{1}{2}\log\det(K(\Btheta_{\idxt})) \Big),
\label{eq:gp_hypers_map}
\end{align}
where $\pi(\Bphi)$ is the prior density for GP hyperparameters and where the covariance function in $K(\Btheta_{\idxt})=k(\Btheta_{\idxt},\Btheta_{\idxt}) + \sigma_n^2\Id$ depends naturally also on $\Bphi$. For the rest of the paper, we assume that the MAP estimate is used for GP hyperparameters, however, we also briefly discuss how one could integrate over them in Section \ref{subsec:gp_hypers}.

\subsection{Efficient parameter acquisition} \label{subsec:acquisition_rule}


We define our loss function $\mathcal{L}_{\piabc}$ for model-based ABC inference as
%
%
\begin{equation}
\mathcal{L}_{\piabc} (D_{\idxt})
= \int_{\Theta} \Var(\tildepiabc(\Btheta) \cond D_{\idxt}) \ud \Btheta 
= \int_{\Theta} \pi^2(\Btheta) \Var(\pacc(\Btheta) \cond D_{\idxt}) \ud \Btheta, \label{eq:abc_loss}
\end{equation}
where $\tildepiabc(\Btheta) = \pi(\Btheta) \pacc(\Btheta)$ is the unnormalised ABC posterior and the variance is taken with respect to the unknown latent function $f$ conditioned on the training data $D_{\idxt}$. 
We call the function in Equation (\ref{eq:abc_loss}) as the integrated variance loss function. It measures the uncertainty in the unnormalised ABC posterior density averaged over the parameter space $\Theta$. 
The loss function is defined in terms of the unnormalised ABC posterior because we are here interested in minimising the uncertainty in the posterior shape. Also, this choice allows tractable computations unlike some other potential choices such as defining the integrated variance over the normalised ABC posterior density function. However, in principle, other loss functions, suitable for particular problem at hand, could be defined. 

We obtain the following formula for computing the expected integrated variance loss function $L_{\idxt}(\Btheta^{*})$ (abbreviated as ``\expintvar{}'') when the new candidate evaluation location is $\Btheta^{*}$. The proof is rather technical and can be found in the supplementary. 
%
%
\begin{proposition} \label{prop:expintvar}
Under the GP model described in Section \ref{subsec:gp_model}, the expected integrated variance after running the simulation model with parameter $\Btheta^{*}$ is given by
\begin{align}
L_{\idxt}(\Btheta^{*}) 
&= \mean_{\Delta^{*}\cond \Btheta^{*}, D_{\idxt}} \int_{\Theta} \pi^2(\Btheta) \Var(\pacc(\Btheta) \cond \Delta^{*}, \Btheta^{*}, D_{\idxt}) \ud \Btheta \\
\begin{split}
&= 2\int_{\Theta} \pi^2(\Btheta) \Bigg[ T\left(\frac{\epsilon - m_{\idxt}(\Btheta)}{\sqrt{\sigma_n^2 + v_{\idxt}^2(\Btheta)}}, 
\sqrt\frac{\sigma_n^2 + v_{\idxt}^2(\Btheta) - \Deltav_{\idxt}^2(\Btheta,\Btheta^{*})}{\sigma_n^2 + v_{\idxt}^2(\Btheta) + \Deltav_{\idxt}^2(\Btheta,\Btheta^{*})}\right) 
\\ 
&\quad- T\left(\frac{\epsilon - m_{\idxt}(\Btheta)}{\sqrt{\sigma_n^2 + v_{\idxt}^2(\Btheta)}}, 
\frac{\sigma_n}{\sqrt{\sigma_n^2 + 2v_{\idxt}^2(\Btheta)})}\right) \Bigg] \ud \Btheta, 
\label{eq:expintvar}
\end{split}
\end{align}
where the variance of $\pacc(\Btheta)$ is taken with respect to $\probmeas(\ud f \cond \Delta^{*}, \Btheta^{*}, D_{\idxt})$, the function $T(\cdot,\cdot)$ is the Owen's t-function as in Equation (\ref{eq:owens_t}) and
\begin{align}
\Deltav_{\idxt}^2(\Btheta,\Btheta^{*}) = 
\frac{\cov_{\idxt}^2(\Btheta,\Btheta^{*})}{\sigma_n^2 + v_{\idxt}^2(\Btheta^{*})},
\label{eq:dgpvar}
\end{align}
where $\cov_{\idxt}(\Btheta,\Btheta^{*}) = k(\Btheta,\Btheta^{*}) - k(\Btheta,\Btheta_{\idxt})K^{-1}(\Btheta_{\idxt})k(\Btheta_{\idxt},\Btheta^{*})$
is the posterior covariance between the evaluation point $\Btheta$ and the candidate location for the next evaluation $\Btheta^{*}$. 
\end{proposition}

A future evaluation at $\Btheta^{*}$ causes a deterministic reduction of the GP variance that is given by the Equation (\ref{eq:dgpvar}). 
However, the variance of the unnormalised ABC posterior depends on the realisation of the discrepancy $\Delta^{*}$ and we need to average over $\pi(\Delta^{*}\cond \Btheta^{*}, D_{\idxt})$. 
%
It is easy to see that if $\Deltav_{\idxt}^2(\Btheta,\Btheta^{*}) \rightarrow v_{\idxt}^2(\Btheta)$, then the integrand at the corresponding parameter $\Btheta$ approaches zero. It can also be shown (using \citet[Eq.~2.3]{Owen1980}) that if $\Deltav_{\idxt}^2(\Btheta,\Btheta^{*}) = 0$, then the integrand in Equation (\ref{eq:expintvar}) equals the current variance given by Equation (\ref{eq:post_var}).

While some of the derivations could be done analytically, computing the expected integrated variance requires integration over the parameter space $\Theta$. This can be done with Monte Carlo or quasi-Monte Carlo methods; here we use importance sampling (IS) to approximate the integral when $p>2$. Using the IS estimator \citep[Eq.~3.10, p.~95]{Robert2004}, we obtain 
\begin{align}
L_{\idxt}(\Btheta^{*}) &= 2 \int_{\Theta} \pi^2(\Btheta) g_{\idxtp}(\Btheta,\Btheta^{*}) \ud \Btheta 
%
\approx 2 \sum_{i=1}^{s} \omega^{(i)} \pi^2(\Btheta^{(i)}) g_{\idxtp}(\Btheta^{(i)},\Btheta^{*}), \label{eq:is}
%
%
\end{align}
where $g_{\idxtp}(\Btheta,\Btheta^{*})$ is the term inside the square brackets in Equation (\ref{eq:expintvar}) and the importance weights are given by
\begin{equation}
\omega^{(i)} = \frac{1}{\pi^2(\Btheta^{(i)})\Var(\pacc(\Btheta^{(i)}) \cond D_{\idxt})} \Bigg/ \sum_{j=1}^{s} \frac{1}{\pi^2(\Btheta^{(j)})\Var(\pacc(\Btheta^{(j)}) \cond D_{\idxt})}, 
\label{eq:is_weights}
\end{equation}
and $\Btheta^{(i)} \sim \pi_q(\cdot)$ for $i=1,\ldots,s$. The importance distribution $\pi_q(\Btheta)$ is proportional to the prior squared times the current variance of the unnormalised ABC posterior i.e.~$\pi_q(\Btheta) \propto \pi^2(\Btheta) \Var(\pacc(\Btheta) \cond D_{\idxt})$. 
%
This importance distribution is a reasonable choice, because one evaluation is unlikely to change the variance surface much and the expected variance thus has similar shape as the current variance surface.
%
It is easy to see that if the prior is bounded and proper i.e.~$\pi(\Btheta) < \infty$ and $\int_{\Theta}\pi(\Btheta)\ud\Btheta=1$, then $\pi_q$ defines a valid probability density function (up to normalisation). 
Because the normalising constant of $\pi_q$ is unavailable, we need to normalise the weights in Equation (\ref{eq:is_weights}).  
Generating samples from the importance distribution $\pi_q$ is not straightforward but can be done using (e.g.~adaptive) Metropolis algorithm or using sequential Monte Carlo methods. 

As outlined in Section \ref{sec:problem}, the new evaluation location is chosen to minimise the expected loss, that is
\begin{equation}
\Btheta_{t+1} \in \{\Btheta\in\Theta: \Btheta = \arg\min_{\Btheta^{*} \in \Theta} L_{\idxt}(\Btheta^{*})\}, 
\label{eq:acq_opt}
\end{equation}
where the right hand side is a set of parameters because the minimiser may not be unique. We call this new strategy 'an acquisition rule' according to the nomenclature in the Bayesian optimisation literature. Unlike in BO, however, our aim is not to optimise the discrepancy but to minimise our uncertainty in the ABC posterior approximation. The second term in Equation (\ref{eq:expintvar}) does not depend on $\Btheta^{*}$ and its value can be computed just once (or omitted completely) and the normalisation of the prior density $\pi(\Btheta)$ does not affect the solution of (\ref{eq:acq_opt}) as it only scales the objective function. Gradient-based optimisation with multiple starting points can be used for solving (\ref{eq:acq_opt}) and the gradient is derived in the supplementary material. 
The resulting algorithm for estimating the ABC posterior is outlined as Algorithm \ref{alg:gp_abc_alg}. 

\begin{algorithm}
\caption{GP-based ABC inference using the expected integrated variance acquisition function. \label{alg:gp_abc_alg}} 
 \begin{algorithmic}[1]
 %
 %
 \State Generate initial training locations $\Btheta_{1:t_0} \sim \pi(\cdot)$ 
 \For{$t=1:t_0$}
 \State Simulate $\Bx_{t} \sim \pi(\cdot \cond \Btheta_{t})$ 
 \State Compute $\Delta_{t} \leftarrow \Delta(\Bx_{obs},\Bx_{t})$
 \EndFor
 \For{$t=t_0:t_{\text{max}}-1$}
  %
  \State Estimate GP hyperparameters $\Bphi^{\text{MAP}}_{\idxt}$ using $D_{1:t}$ and Equation (\ref{eq:gp_hypers_map}) 
  %
  \State Precompute Cholesky factorisation for the GP prediction
  \State Simulate evaluation points $\Btheta^{(i)}$ and weights $\omega^{(i)}$ for $i = 1,\ldots,s$ by sampling from $\pi_q(\cdot)$ 
  \State Precompute the second term in Equation (\ref{eq:expintvar}) 
  %
  \State Obtain $\Btheta_{t+1}$ by solving the optimisation problem in Equation (\ref{eq:acq_opt})
  \State Simulate $\Bx_{t+1} \sim \pi(\cdot \cond \Btheta_{t+1})$ 
  \State Compute $\Delta_{t+1} \leftarrow \Delta(\Bx_{obs},\Bx_{t+1})$
  \State Update the training data $D_{1:t+1} \leftarrow D_{1:t} \cup \{(\Delta_{t+1}, \Btheta_{t+1})\}$
 \EndFor
 \State Estimate GP hyperparameters $\Bphi^{\text{MAP}}_{1:t_{\text{max}}}$ using $D_{1:t_{\text{max}}}$ and Equation (\ref{eq:gp_hypers_map})
 \State Simulate samples $\Bvartheta^{(1:n)}$ from the density defined by Equation (\ref{eq:post_mean})
 \State \Return $\Bvartheta^{(1:n)}$ as a sample from the approximate posterior density
 \end{algorithmic}
\end{algorithm}

\subsection{Alternative acquisition rules} \label{subsec:alternatives}

We briefly discuss some alternative acquisition rules for ABC inference. Their derivations follow directly from our previous analysis and we include these strategies in our experiments in Section \ref{sec:experiments}. 
One such alternative to the expected integrated variance strategy is to evaluate where the current uncertainty of the unnormalised ABC posterior is highest. This approach is similar to \citet{Kandasamy2015}. This strategy is a reasonable heuristic in the sense that the next evaluation location is where improvement in estimation accuracy is needed most, although it does not account for how large an improvement can be expected at the location, or overall. 
This approach requires solving the optimisation problem 
\begin{align}
\Btheta_{t+1} &\in \{\Btheta\in\Theta: \Btheta = \arg\max_{\Btheta^* \in \Theta} \pi^2(\Btheta^*)\Var(\pacc(\Btheta^*) \cond D_{1:t}) \} \\
&= 
\left\{\Btheta\in\Theta: \Btheta = \arg\max_{\Btheta^* \in \Theta} \left( \log\pi(\Btheta^*) + \log\sqrt{\Var(\pacc(\Btheta^*) \cond D_{1:t})} \right) \right\}, \label{eq:maxvar_acq_def_prior}
\end{align}
where the current variance $\Var(\pacc(\Btheta) \cond D_{1:t}))$ is given by Equation (\ref{eq:post_var}) 
and is taken with respect to $\Pi(\ud f \cond D_{\idxt})$. We call this method the ``\maxvar{}'' acquisition rule. 
%
The gradient of this acquisition function is derived in the supplementary material.


To encourage further exploration, similarly to \citet{Gutmann2015}, we also consider a stochastic variant of the \maxvar{} acquisition rule in Equation (\ref{eq:maxvar_acq_def_prior}). Specifically, we generate the evaluation point randomly according to the variance surface $\pi_q(\Btheta) \propto \pi^2(\Btheta)\Var(\pacc(\Btheta) \cond D_{1:t})$ which we also use as an importance distribution for the expected integrated variance acquisition function as discussed earlier. That is, instead of finding the maximiser, we generate $\Btheta_{t+1} \sim \pi_{q}(\Btheta)$. 
%
This strategy requires generating random samples from $\pi_q(\Btheta)$ but sampling (and optimising) the variance function can be done fast compared to the time required to run the simulation model. We call this method ``\randmaxvar{}''.

The stochastic acquisition rule is reminiscent of Thompson sampling, 
but it is actually quite different. In our method, acquisitions are drawn at random from the probability distribution which is proportional to the (point-wise) variance of the approximate posterior density. In Thompson sampling, instead, one generates a posterior density realisation from the model, and chooses the next point as the maximiser of this realisation.

The \maxvar{} and \randmaxvar{} strategies avoid the integration over the parameter space that is necessary for the \expintvar{} method. 
However, one could replace the integration in \expintvar{} by only a single evaluation at the candidate point. In other words, this method chooses a location with the highest expected reduction in the uncertainty of the unnormalised ABC posterior in that particular location. We call this variant the ``\expdiffvar{}'' from now on.

A comparison of the acquisition functions in a one-dimensional toy problem is shown in Figure \ref{fig:acq_comparison_supplementary}. A simulation model has been run eight times and the acquisition functions for selecting the ninth evaluation location (shown with triangles) are plotted for comparison. The current variance surface (\maxvar{}) and the expected integrated variance (\expintvar{}) function are plotted with three values of the threshold $\epsilon$. Unlike the current variance surface, the expected integrated variance appears insensitive to the value of the threshold. 
Figure \ref{fig:acq_comparison_supplementary}b shows also that using the MAP-estimate for the GP hyperparameters causes underestimation of the variance of the unnormalised ABC posterior. In the next section we show how the uncertainty in GP hyperparameters is (approximately) taken into account. 
%
\begin{figure}[ht]
\centering
\includegraphics[width=0.6\textwidth]{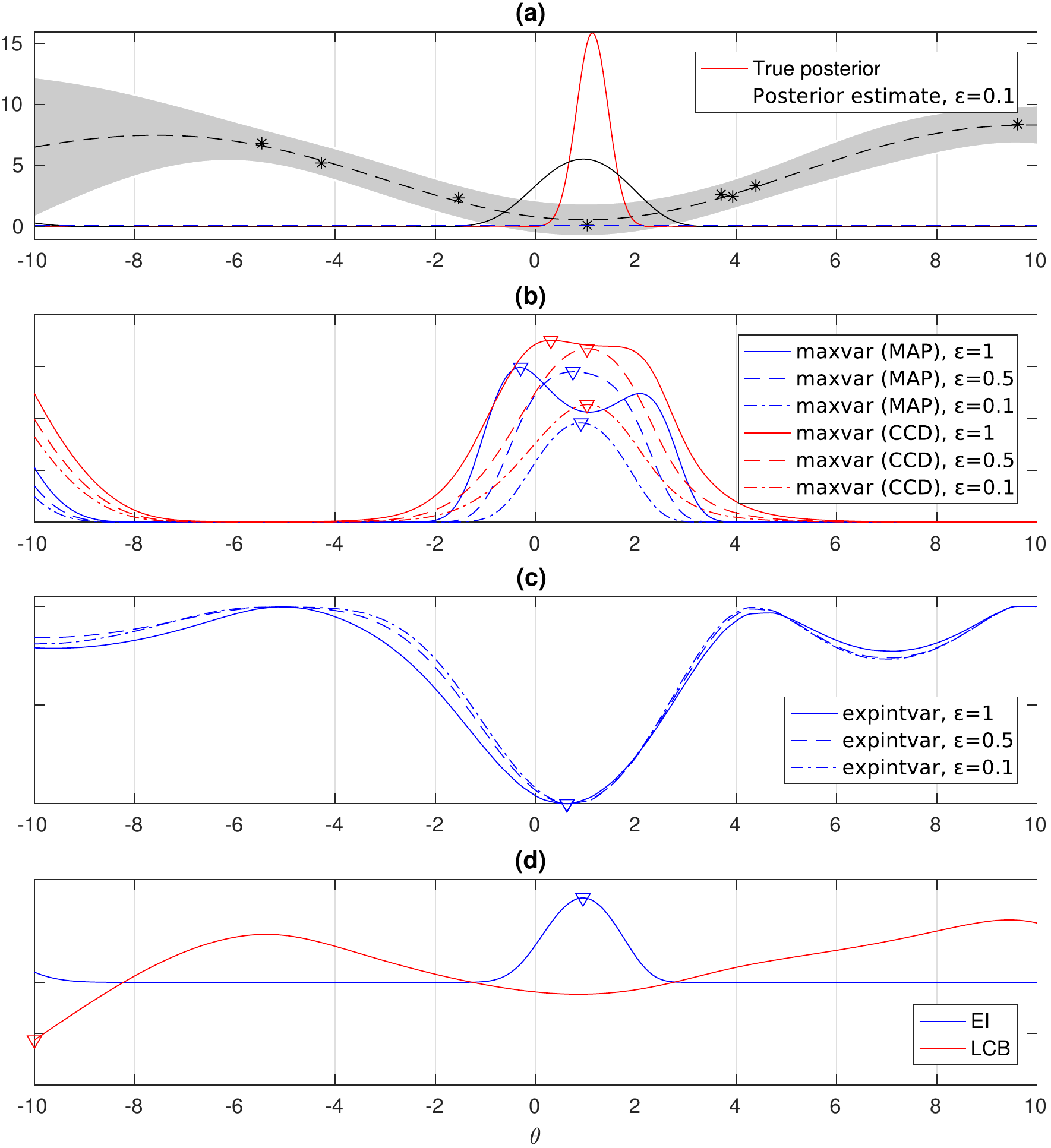}
\caption
{(a) The discrepancy observations (black stars) and the estimate of the ABC posterior density based on eight training data points (with $\epsilon=0.1$) as compared to the true posterior. (b) The variance of the unnormalised ABC posterior is computed using the MAP estimate (\maxvar{} (MAP)) or CCD integration (\maxvar{} (CCD)) for GP hyperparameters and for three values of the threshold $\epsilon$. Details of the CCD integration are in Section \ref{subsec:gp_hypers}. (c) Expected integrated variance (\expintvar{}) acquisition function. (d) Expected improvement (EI) and lower confidence bound (LCB) criteria (scaled to fit the same figure). Note that the scales of the variance function (b) and acquisition functions computed with different thresholds, (c) and (d), are not comparable. 
} \label{fig:acq_comparison_supplementary}
\end{figure}

\subsection{Uncertainty in hyperparameters} \label{subsec:gp_hypers}

Above we assumed that either the GP hyperparameters $\Bphi$ are known or MAP estimates are used. Here we briefly discuss how the uncertainty in the GP hyperparameters could also be taken into account. 
Integrating over the uncertainty in the GP hyperparameters requires Monte Carlo sampling as in \citet{Murray2010}. An alternative approach is to use central composite design \citep{Rue2009,Vanhatalo2010}.
Briefly, in central composite design (CCD) certain design points $\Bphi^i$ are chosen and each of them is given a weight $\omega^i \propto \pi(\Bphi^i \cond D_{\idxt})\gamma^i \propto \pi(D_{\idxt} \cond \Bphi^i)\pi(\Bphi^i)\gamma^i$, where $\gamma^i$ is a design weight. This approach has the advantage that the amount of design points grows only moderately with increased dimension and has been shown to yield good accuracy in practice. Further details on choosing the design points and their weights are given in \citet{Vanhatalo2010}.

Integrating over the uncertainty in the GP hyperparameters $\Bphi$ in Lemma \ref{lemma:p_mean_var} leads to the following calculations. Using the law of total expectation yields 
\begin{align}
%
%
\mean(\pacc(\Btheta)) = \mean_{\Bphi}\mean_{f\cond\Bphi}(\pacc(\Btheta)) 
\approx \sum_{i} \omega^i \, \mean_{f\cond\Bphi=\Bphi^i}(\pacc(\Btheta)) 
= \sum_{i} \omega^i \, \Phi\left(a(\Btheta,\Bphi^i) \right), \label{eq:int_mean_p}
\end{align}
where the grid points and the corresponding weights are $\Bphi^i$ and $\omega^i$, respectively, and where $a(\Btheta,\Bphi^i) = {(\epsilon - m_{\idxt}(\Btheta\cond \Bphi^i))}/{\sqrt{(\sigma_n^2)^i + v_{\idxt}^2(\Btheta\cond \Bphi^i)}}$. 
If Monte Carlo sampling is used, then $\omega^i = 1/s$ for all $i=1,\ldots,s$, where $s$ is the number of samples.  
Similarly, for the variance we obtain
\begin{align}
\Var(\pacc(\Btheta)) &= \mean(\pacc(\Btheta)^2) - [\mean(\pacc(\Btheta))]^2 \nonumber \\ 
%
%
&= \mean_{\Bphi}\mean_{f\cond\Bphi}(\pacc(\Btheta)^2) - [\mean_{\Bphi}\mean_{f\cond\Bphi}(\pacc(\Btheta))]^2 \nonumber \\ 
&\approx \sum_{i} \omega^i \left[ \Phi(a(\Btheta,\Bphi^i)) - 2T(a(\Btheta,\Bphi^i), b(\Btheta,\Bphi^i)) \right] 
- \left[ \sum_{i} \omega^i \, \Phi(a(\Btheta,\Bphi^i)) \right]^{2},
\label{eq:int_var_p} 
\end{align}
where $b(\Btheta,\Bphi^i) = {(\sigma_n)^i}/{\sqrt{(\sigma_n^2)^i + 2v_{\idxt}^2(\Btheta\cond \Bphi^i)}}$. 
This formula with CCD integration was already used in Figure \ref{fig:acq_comparison_supplementary}b.

One can also take into account the uncertainty in GP hyperparameters in the expected integrated variance acquisition function. The posterior predictive distribution for a future simulation is then approximated by a Gaussian mixture and one can make the simplification by (incorrectly) assuming that the future evaluation will not affect the GP hyperparameters but only the latent function $f$. Evaluating the resulting acquisition function requires a large number of calls to Owen's t-function and GP formulas and is computationally more costly and thus possibly impractical. 
Alternatively, one can define an integrated\footnote{Note that the term ``integrated'' here refers to integrating over GP hyperparameters $\Bphi$ and not for integrating over the parameter space $\Theta$ as in Equation (\ref{eq:expintvar}).} acquisition function as in 
\citet{Snoek2012,HernandezLobato2014,Wang2017} which only requires computing Equation (\ref{eq:expintvar}) for each sampled GP hyperparameter $\Bphi^i$. The integrated acquisition function is then averaged over these values. 
Alternatively, one could simply use the posterior mean of the hyperparameters in the place of the MAP estimate. 
However, we leave a detailed analysis for future work.

\section{Experiments} \label{sec:experiments}

We compare the proposed expected integrated variance acquisition rule (\expintvar{}) to commonly used BO strategies: expected improvement (EI) and lower confidence bound (LCB) criterion, see e.g.~\citet{Shahriari2015}. We use the same trade-off parameter for LCB as \citet{Gutmann2015}, but unlike them, we consider the deterministic LCB rule. 
As a simple baseline, we also draw points sequentially from the uniform distribution, abbreviated as ``unif''. 
We also included the probability of improvement (PI) strategy in preliminary experiments, but it resulted in poor estimates and was therefore excluded from the comparisons. 

In addition to \expintvar{}, we also include the \maxvar{}, \randmaxvar{} and \expdiffvar{} strategies, which were briefly described in Section \ref{subsec:alternatives}, to our list of methods to be compared. 
The MAP estimate for the GP hyperparameters $\Bphi$ is used in all the experiments.
%
We use MATLAB and GPstuff 4.6 \citep{Vanhatalo2013} for GP fitting. For fast and accurate computation of Owen's t-function, we use a C-implementation of the algorithm by \citet{Patefield2000}. 
The algorithms in this article are also made available in the ELFI (engine for likelihood-free inference) Python software package by \citet{Lintusaari2017}. 

The total variation (TV) distance is used for assessing the accuracy of the posterior approximation. It is defined as $\text{TV} = 1/2 \int_{\Theta} |\hatpiabc(\Btheta) - \pitrue(\Btheta)| \ud \Btheta $, where $\hatpiabc$ is the estimated ABC posterior and $\pitrue$ is the reference distribution. As the reference we use the exact ABC posterior with the same threshold as used for the approximations but in some scenarios, however, the reference distribution is the exact posterior for computational convenience and to see the overall approximation quality. 
The point estimate of the ABC posterior density function for the comparisons is always computed using Equation (\ref{eq:post_mean}). {\color{\revcol} We demonstrate our approach with multiple toy models as well as two realistic models. While the likelihood is actually available for the toy models, we restrict our comparison to the model-based ABC methods, the focus of this work, at the same time acknowledging that in practice with likelihood available the standard methods, such as MCMC, are expected to outperform the likelihood-free alternatives. 
An overview of the results is given in Table~\ref{table1} and discussed in detail in the following sections. 
}

\begin{table} 
\begin{center}
  \begin{tabular}{ lccccccc }
    \toprule
    & \expintvar{} & \expdiffvar{} & \maxvar{} & \randmaxvar{} & LCB & EI & unif \\
    \midrule
    unimodal$^*$ &    1.00 &   1.39  &  1.52  &  1.23  &  1.27  &  2.54  &  \textbf{0.97}\\
    bimodal$^*$ & \textbf{1.00} &   1.23  &  1.24  &  1.03  &  1.04  &  1.51  &  1.13\\
    unidentifiable$^*$ & \textbf{1.00}  &  1.11  &  1.21 &   1.12  &  1.01  &  1.58  &  1.49\\
    banana$^*$ & \textbf{1.00}  &  1.12 &   1.23  &  1.09  &  1.08  &  1.67 &   1.47\\
    \midrule
Gaussian (strong prior)  &\textbf{1.00}&    1.13&    1.18&    1.24&    1.68&    2.68&    1.02\\
Gaussian (weak prior)   &\textbf{1.00}&    1.20&    1.16&    1.07&    1.10&    1.59&    1.83\\
    \midrule
    Gaussian 3d & 1.00&    1.26&    1.16&    \textbf{0.94}&    1.26&    1.67&    2.24  \\ 
    Gaussian 6d & 1.00&    1.06&    1.08&    \textbf{0.98}&    1.14&    1.42&    1.94  \\ 
    Gaussian 10d & \textbf{1.00}&    1.08&    1.08&    1.17&    1.21&    1.51&    1.45 \\ 
    \midrule
    Lotka-Volterra & \textbf{1.00}&    1.20&    1.37&    1.10&    1.15&    1.85&    1.62 \\
    \bottomrule
  \end{tabular}
  \caption{{\color{\revcol} Results for the test problems. The numbers in the table represent the median of the area under the TV curve (TV values as a function of iteration) scaled so that the proposed \expintvar{} method obtains value one. Smaller values mean better average performance. In the first four test problems (marked with $^*$), the reference distribution is the exact ABC posterior obtained using the same threshold as the model-based estimate. In the other cases, TV distance is computed with respect to the 'true' posterior. For the Gaussian 3d-10d examples, the TV represents the average TV of marginal densities. 
  }} \label{table1}
\end{center}
\end{table}

\subsection{Synthetic 2D simulation models} \label{sec:synth_experiments}

To compare the different acquisition strategies first without the need to actually handle different simulation models, we construct ``synthetic'' discrepancies by adding Gaussian noise to certain parametric curves, and use these to simulate the discrepancy realisations directly. The {\color{\revcol} exact ABC posterior that is used as a reference distribution here} is computed using the posterior density given by Equation (\ref{eq:correct_estim_post}) with a small predefined threshold $\epsilon$. 
%
%
As test cases we consider 1) a unimodal density with two correlated variables, 2) a bimodal density, 3) a density where the first parameter is (almost) unidentifiable, and 4) a banana shaped density. For all cases, a uniform prior was assumed. The resulting {\color{\revcol} exact ABC} posterior densities are illustrated in Figure \ref{fig:synth_problems}. (See supplementary material for additional details). 
%
%
The integration and sampling steps required by \expintvar{} and \randmaxvar{} strategies are performed in a 2D grid of $50^2$ evaluation locations. 
The initial training set size is $t_0=10$ and the initial training sets for the repeated experiments are generated randomly from the uniform prior as is done in the other test cases as well. 

The threshold is fixed so that differences in approximation quality between the acquisition methods are solely caused by the selection of the evaluation locations. 
However, because selecting a reasonable threshold can be challenging in practice, we also examine how updating this value adaptively during the acquisitions affects the results. In the supplementary we show results when the threshold is constantly updated so that it matches either the $0.01$th or the $0.05$th quantile of the realised discrepancies. 
%

\begin{figure*}[ht]
\centering
\includegraphics[width=0.9\textwidth]{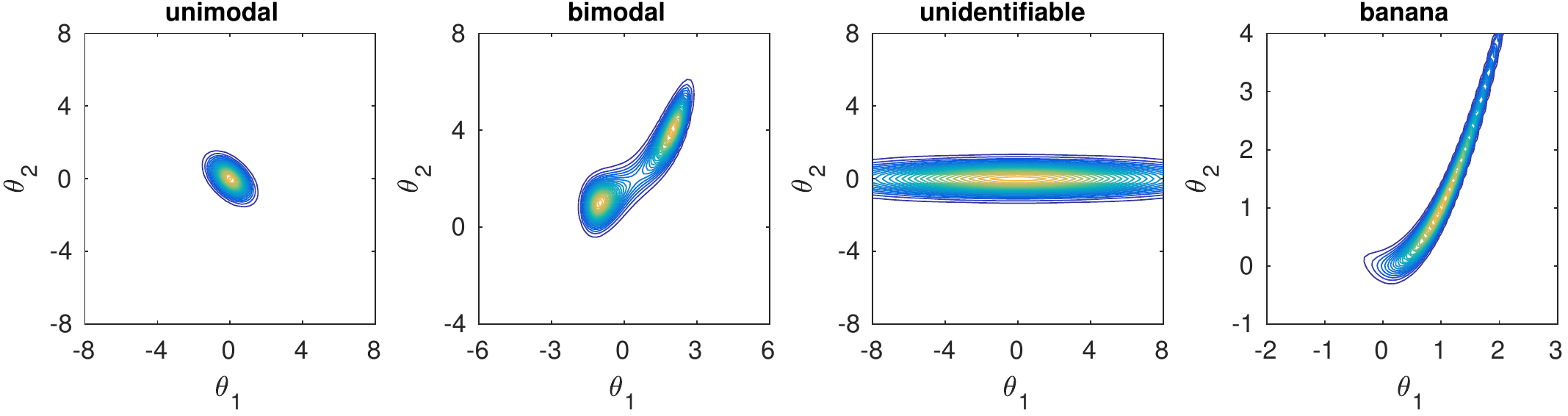}
\caption{Exact ABC posterior densities for the synthetic 2d test problems. 
}
\label{fig:synth_problems}
\end{figure*}

\begin{figure*}[ht]
\centering
\includegraphics[width=0.65\textwidth]{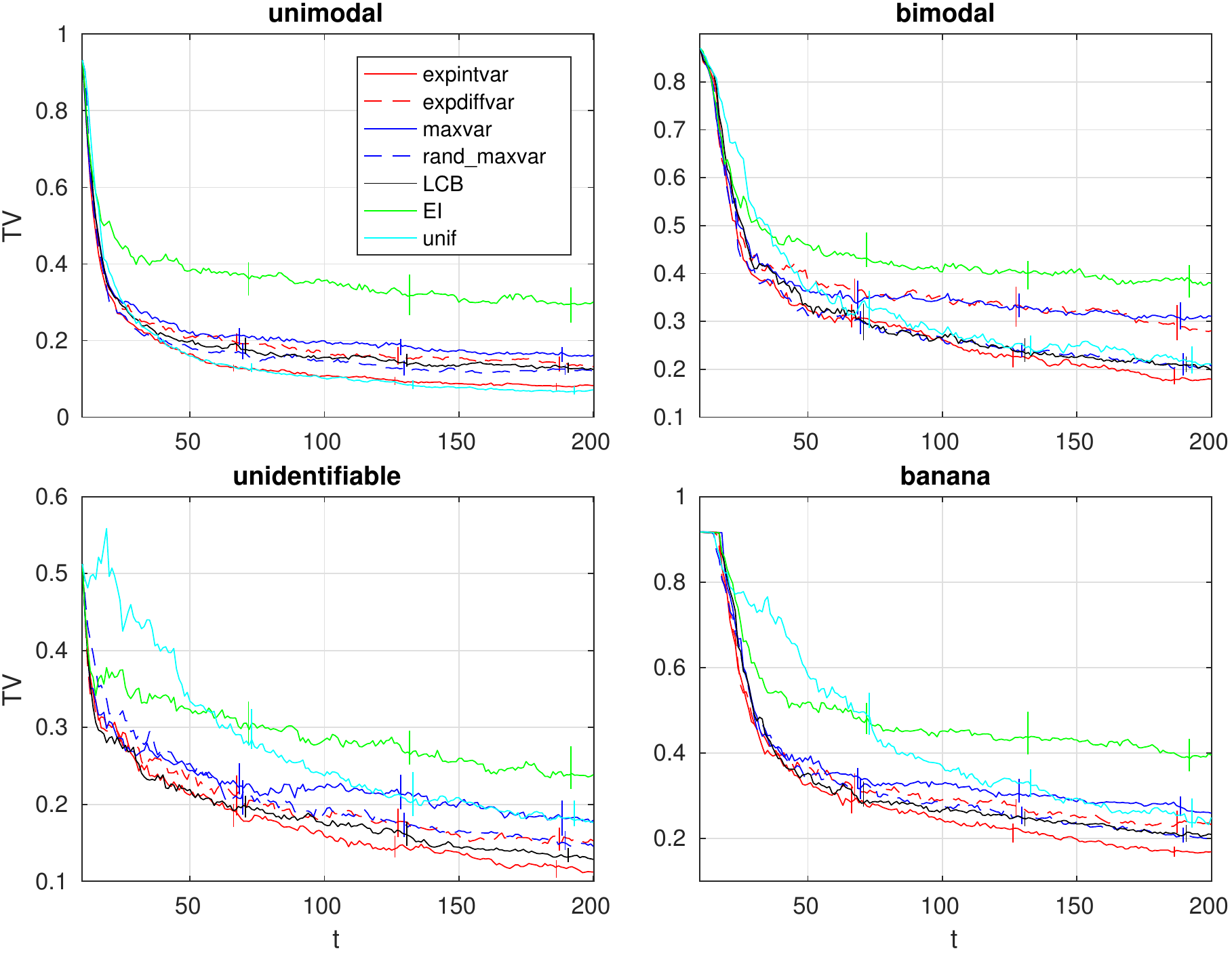}
\caption{Median of the TV distance between the estimated {\color{\revcol} ABC posterior and the corresponding} exact ABC posterior over 100 experiments. Vertical lines show the 95\% confidence interval of the median computed using the bootstrap.}
\label{fig:synth_eps_fixed}
\end{figure*}

The results in Figure \ref{fig:synth_eps_fixed} indicate that the \expintvar{} is the best method overall but also \randmaxvar{} produces good results. Of the common alternatives, LCB is clearly the best and produces results with similar accuracy as \randmaxvar{}. The performance of the EI strategy is poor because it tends to focus evaluations greedily around the mode and samples insufficiently in the tail areas which often results in poor estimates to the ABC posterior density and high variability between different experiments. 
Interestingly, the uniform strategy produces the best estimates in the case of the unimodal example. Most of the acquisitions are not focused on the modal region but because the modelling assumptions hold everywhere and the parameter space is rather small, the extrapolation seems to work well in this case.

\subsection{Gaussian simulation model} \label{sec:gaussian_example}

A simple Gaussian simulation model is used to study the effect of prior strength and the dimension of the parameter space.
Data points are generated independently from $\Bx_i \sim \Normal(\cdot \cond \Btheta,\BSigma), i = 1,\ldots,n$, where $\Btheta \in \Theta = [0,8]^p$ needs to be estimated and the covariance matrix $\BSigma$ is known. If $\Btheta \sim \Normal(\Ba,\BB)$ truncated to $\Theta$, the true posterior is $\Normal(\Btheta \cond \Ba^{\star},\BB^{\star})$ truncated to $\Theta$, where $\Ba^{\star} = \BB^{\star}(\BB^{-1}\Ba + n\BSigma^{-1}\bar{\Bx}_{obs})$, $\BB^{\star} = (\BB^{-1} + n\BSigma^{-1})^{-1}$ and $\bar{\Bx}_{obs} = n^{-1}\sum_{i=1}^n\Bx_{i}$ is the sample mean.
As discrepancy, we use the Mahalanobis distance $\Delta_{\Btheta} = ((\bar{\Bx}_{obs}-\bar{\Bx}_{\Btheta})^T\BSigma^{-1}(\bar{\Bx}_{obs}-\bar{\Bx}_{\Btheta}))^{1/2}$. {\color{\revcol}The true posterior is used for comparisons in all the following experiments with the Gaussian model. }


\subsubsection{Strength of the prior} 

In the first experiment we set $p=2$, $n=5$, $\BSigma_{ii}=1$, and $\BSigma_{ij}=0.5$ for $i\neq j$. The initial training set size is $t_0=10$ and the threshold is fixed to $\epsilon = 0.1$. Integration and sampling for \expintvar{} and \randmaxvar{} are done as in Section \ref{sec:synth_experiments}. The true data mean of $\Btheta$ is $[2,2]^T$. The mean of the (truncated) Gaussian prior is $\Ba = [5,5]^T$ and the covariance matrix is $\BB = b^2\Id$. We vary $b$, allowing us to study the impact of prior strength relative to the likelihood.  
%
Figure \ref{fig:prior} shows the results, and we see that the proposed acquisition rules perform consistently well regardless the strength of prior, and focus the evaluations on the posterior modal region. On the other hand, LCB samples where the discrepancies are small, i.e.~in areas of high likelihood, leading to sub-optimal posterior estimation whenever the prior is also informative. Comparing Figures \ref{fig:prior}a and \ref{fig:prior}b shows that using the \expintvar{} strategy also avoids unnecessary evaluations on the boundary, which is often undesired also in the Bayesian optimisation methods, see \citet{Siivola2017} for a discussion. Curiously, the uniform sampling (unif rule) works well when prior information is strong. 

\begin{figure*}[ht]
\centering
\begin{subfigure}{0.3\textwidth}
\centering
\includegraphics[width=1\textwidth]{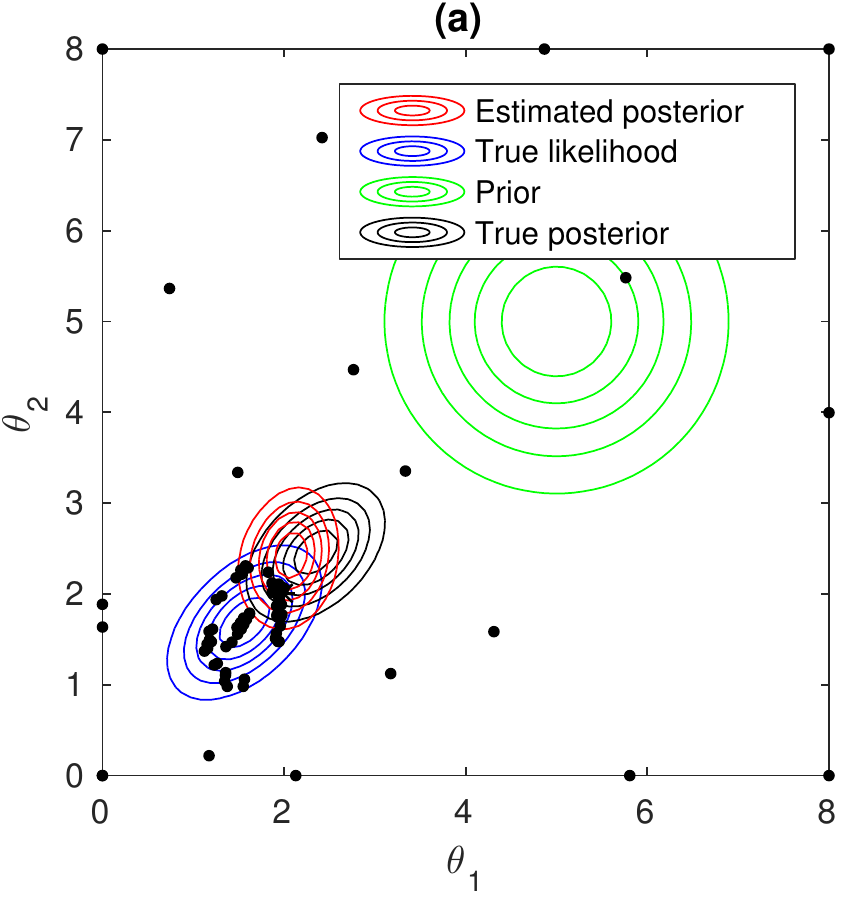}
\end{subfigure}
\begin{subfigure}{0.3\textwidth}
\centering
\includegraphics[width=1\textwidth]{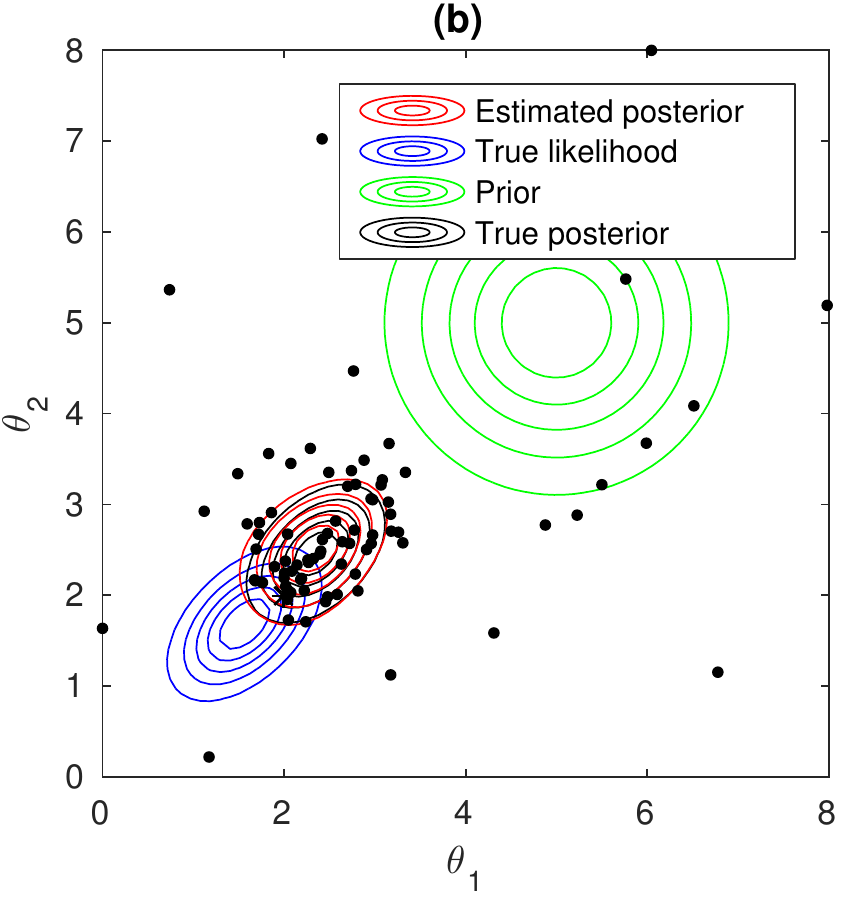}
\end{subfigure}
\begin{subfigure}{0.3\textwidth}
\centering
\includegraphics[width=\textwidth]{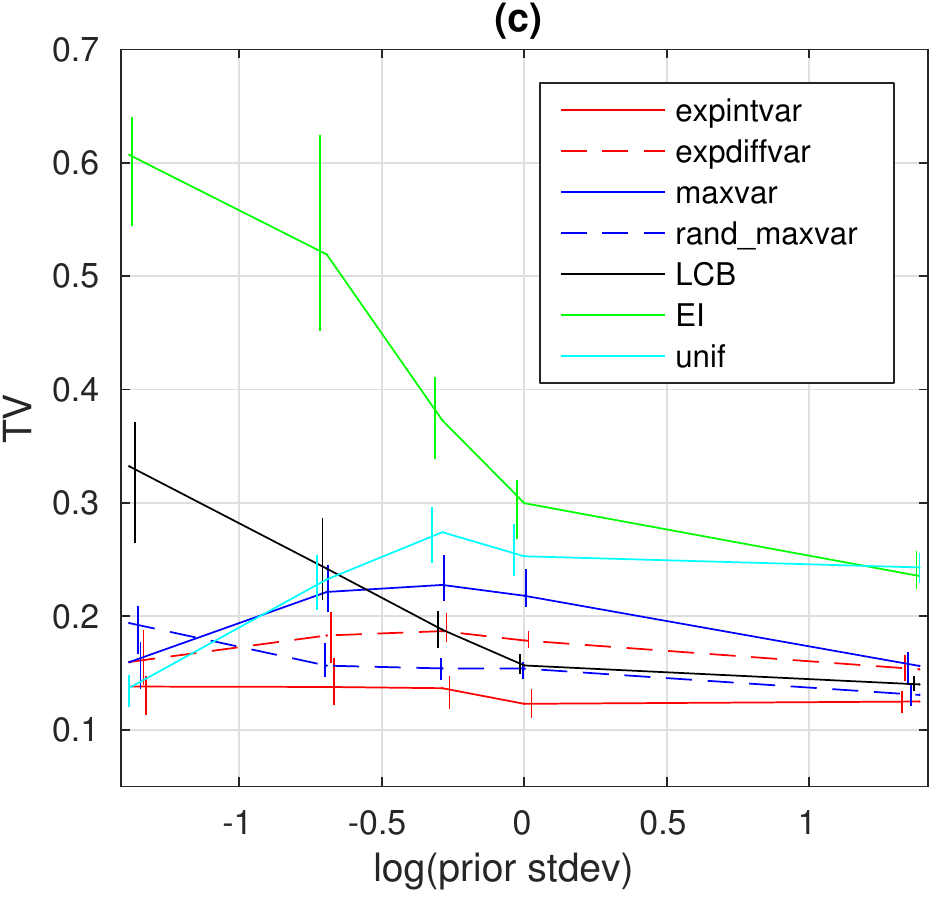}
\end{subfigure}
\caption
{Acquired training data locations (black dots) for (a) LCB, (b) \expintvar{} after $70$ acquisitions. As discussed in \citet{Gutmann2015}, the LCB strategy ignores prior information which here leads to suboptimal selection of evaluation locations. 
(c) Median TV between the estimated {\color{\revcol} ABC posterior and the corresponding} true posterior as a function of the standard deviation (stdev) of the Gaussian prior over $100$ experiments and after $200$ evaluations (small stdev corresponds to strong prior information). 
} \label{fig:prior}
\end{figure*}


\subsubsection{High-dimensional test cases}

Next we investigate the effect of the dimension $p$ of the parameter space. The settings are as before, except that now we use uniform priors supported on $\Theta = [0,8]^p$ 
and the threshold is set adaptively to the 0.01th quantile as described in Section \ref{sec:synth_experiments}. 
Further, $n=15$, and the initial training set sizes are $t_0=20$ (3d) and $t_0=30$ (6d and 10d). Adaptive MCMC (with multiple chains) is used to sample from the {\color{\revcol}model-based ABC posterior estimates required in the line 17 of Algorithm \ref{alg:gp_abc_alg}} and, in the case of \expintvar{} and \randmaxvar{}, from the probability density $\pi_q(\Btheta)$. For \expintvar{} we use $s=500$ importance samples in 3d and $s=200$ in 6d and 10d. Unlike in the other test problems, the TV distance measures here the average of TVs between all the marginal densities. 

Figure \ref{fig:high_dim} shows the results. With $p\leq6$, the \randmaxvar{} is the most accurate and slightly better than \expintvar{} strategy. 
%
However, in 10d it suffers from instability in MCMC convergence. 
Detailed examination shows that the method often produces multimodal posterior estimates which makes the sampling difficult. Such densities are likely a result of the random acquisitions. Namely, even if the uncertainty is high in some region, it can happen that no evaluations occur there during the available iterations, due to the randomness and the curse of dimensionality. EI also tends to produce multimodal difficult-to-sample posterior estimates but similar issues were only rarely observed with other strategies. The results suggest that in high dimensions the strategies that select the acquisition locations deterministically should be preferred over the stochastic ones. 

\begin{figure*}[ht]
\centering
\begin{subfigure}{0.3\textwidth}
\includegraphics[width=\textwidth]{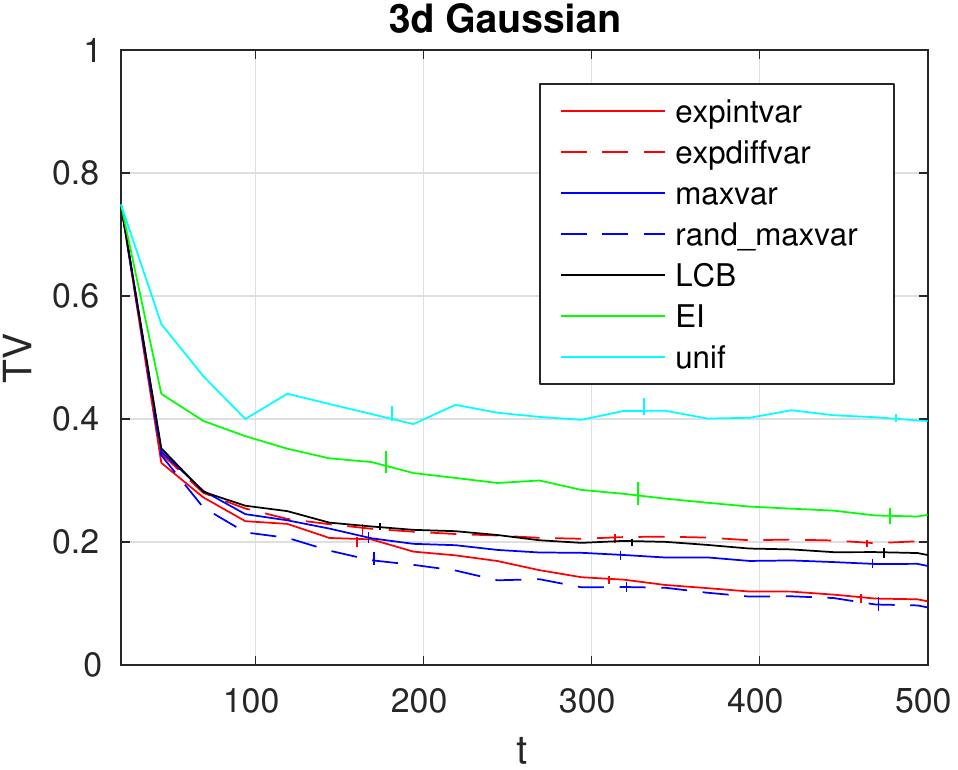}
\end{subfigure}
\begin{subfigure}{0.3\textwidth}
\includegraphics[width=\textwidth]{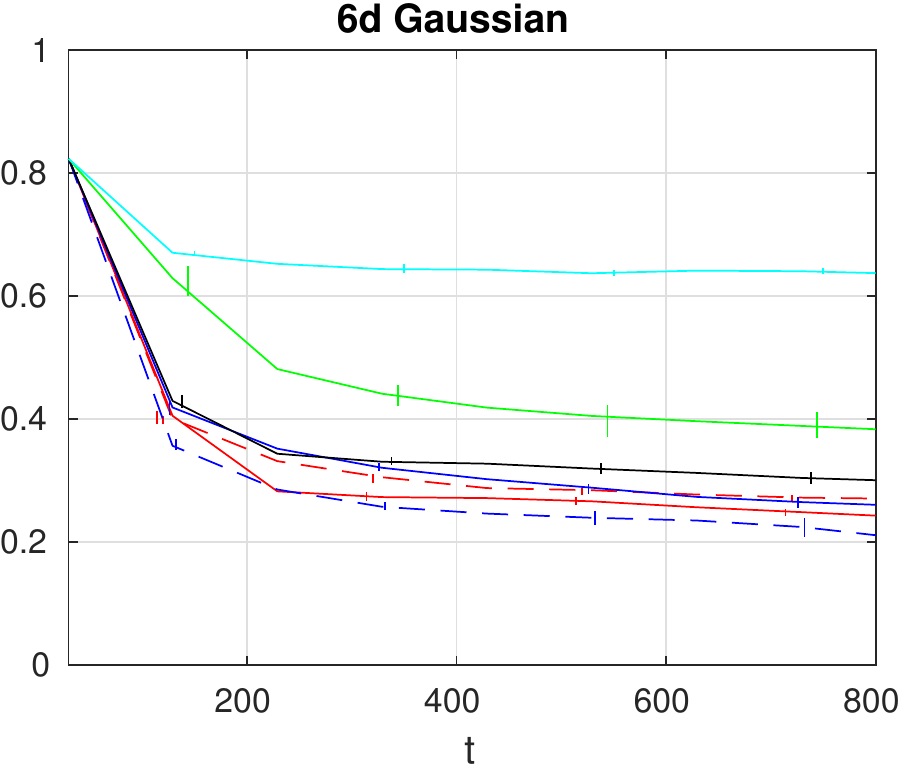}
\end{subfigure}
\begin{subfigure}{0.3\textwidth}
\includegraphics[width=\textwidth]{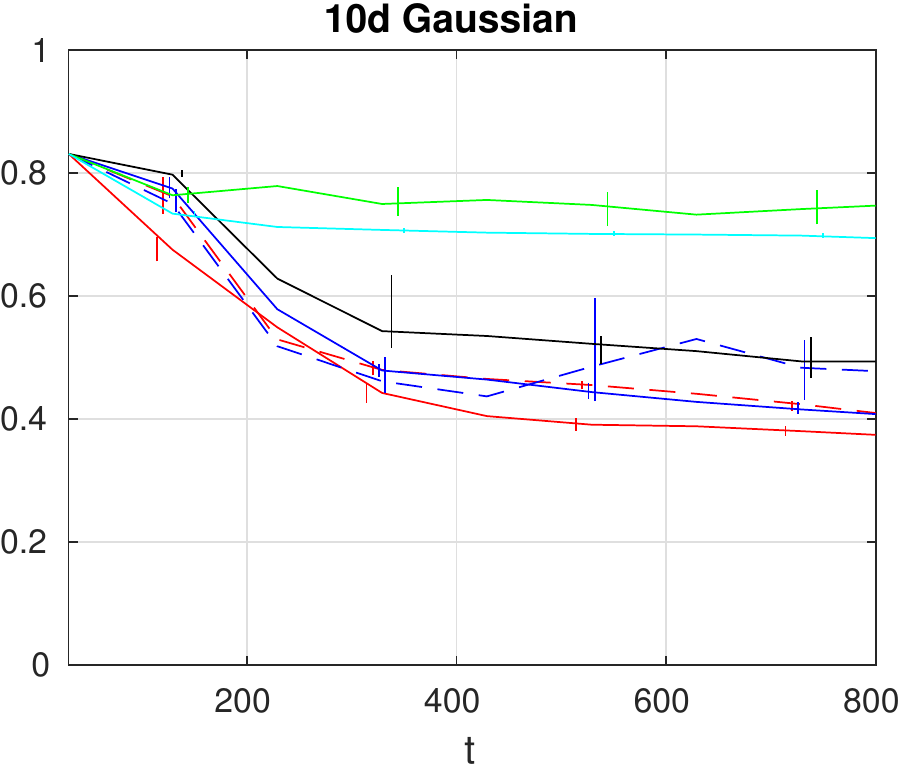}
\end{subfigure}
\caption
{Median of the average marginal TVs between the estimated {\color{\revcol} ABC posterior and the corresponding} true posterior over 100 experiments in the 3d, 6d and 10d Gaussian toy simulation model.
} \label{fig:high_dim}
\end{figure*}


\subsection{Realistic simulation models} \label{sec:real_example}

We consider the Lotka-Volterra model and a model of bacterial infections in day care centers to illustrate the proposed acquisition methods in practical modelling situations.


\subsubsection{Lotka-Volterra model}

The Lotka-Volterra (LV) model \citep{Toni2009} is described by differential equations $x_1'(t) = \theta_1 x_1(t) - x_1(t) x_2(t)$ and $x_2'(t) = \theta_2 x_1(t) x_2(t) - x_2(t)$, where $x_1(t)$ and $x_2(t)$ describe the evolution of prey and predator populations as a function of time $t$, respectively, and $\Btheta=(\theta_1,\theta_2)$ is the unknown parameter to be estimated. 
We use a similar experiment design as in \citet{Toni2009} but with discrepancy {\color{\revcol}$\Delta_{\Btheta} = \log\sum_{ij}(x_j^{\text{obs}}(t_i) - x_j^{\text{mod}}(t_i,\Btheta))^2$}, where $x_j^{\text{obs}}(t_i)$ for $j\in\{1,2\}${\color{\revcol}, $i\in\{1,\ldots,8\}$} denote the noisy observations at times $t_i$, and $x_j^{\text{mod}}(t_i,\Btheta)$ are the corresponding predictions. {\color{\revcol} Up to the log transformation, this is the same as used by \citet{Toni2009}. We also experimented with another discrepancy, where the squared differences were replaced by absolute differences; however, the results were similar.} In comparisons we use the uniform prior with support on $[0,5]^2$, and the reference is the exact posterior distribution that can be computed analytically. We set $t_0 = 10$. The threshold is set to match the smallest observed discrepancy realisations and the integration and sampling required by \expintvar{} and \randmaxvar{} are done as in Section \ref{sec:synth_experiments}. 

The results are presented in Figure \ref{fig:lv}. 
We see that the \expintvar{} strategy produces the best posterior mean estimates (Figure \ref{fig:lv}a), while the best posterior variance estimates are obtained by LCB and \randmaxvar{} (Figure \ref{fig:lv}b). However, the \expintvar{} strategy clearly produces the most accurate posterior approximations in terms of TV distance followed by \randmaxvar{} and the LCB strategy (Figure \ref{fig:lv}c). 

\begin{figure*}
\centering
\begin{subfigure}{0.32\textwidth}
\includegraphics[width=\textwidth]{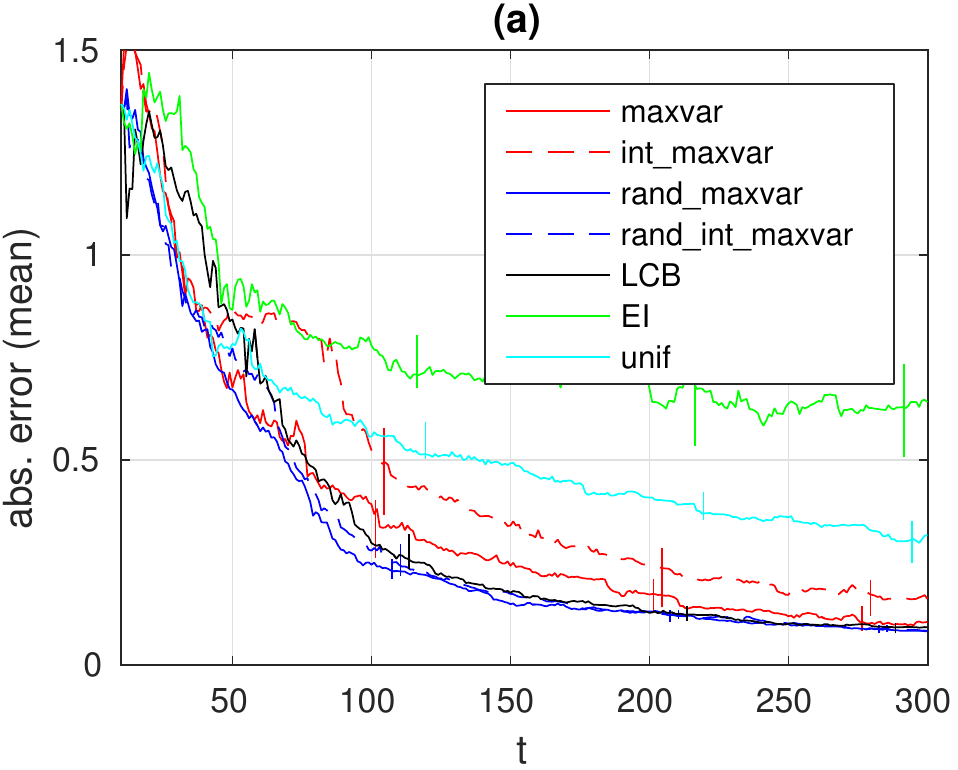}
\end{subfigure}
\begin{subfigure}{0.32\textwidth}
\includegraphics[width=\textwidth]{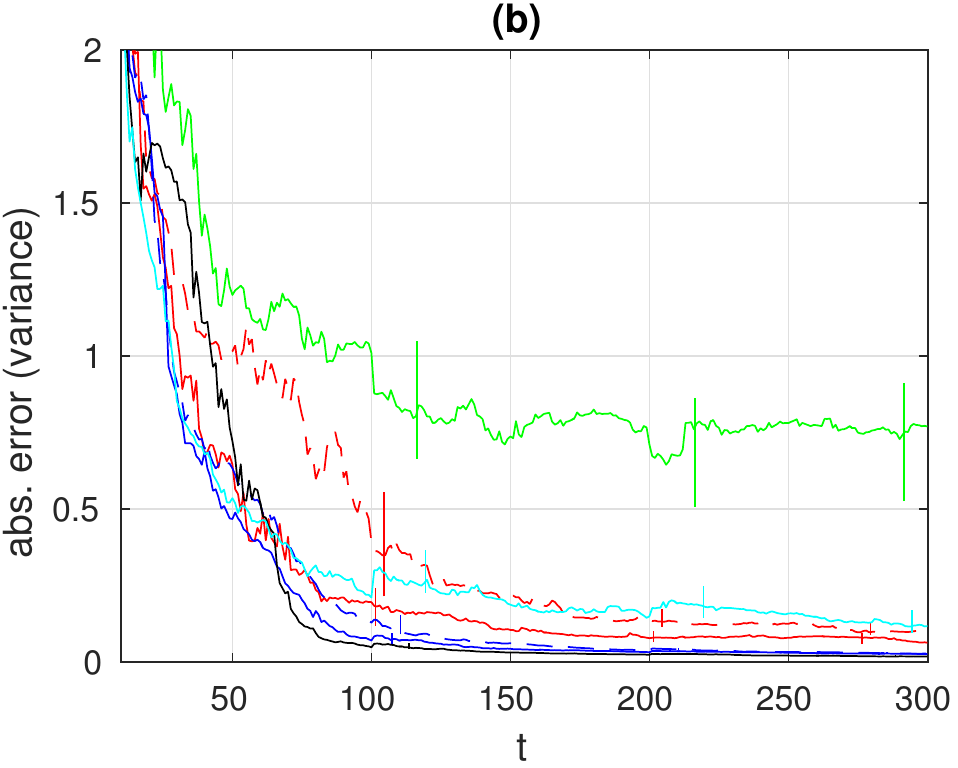}
\end{subfigure}
\begin{subfigure}{0.32\textwidth}
\includegraphics[width=\textwidth]{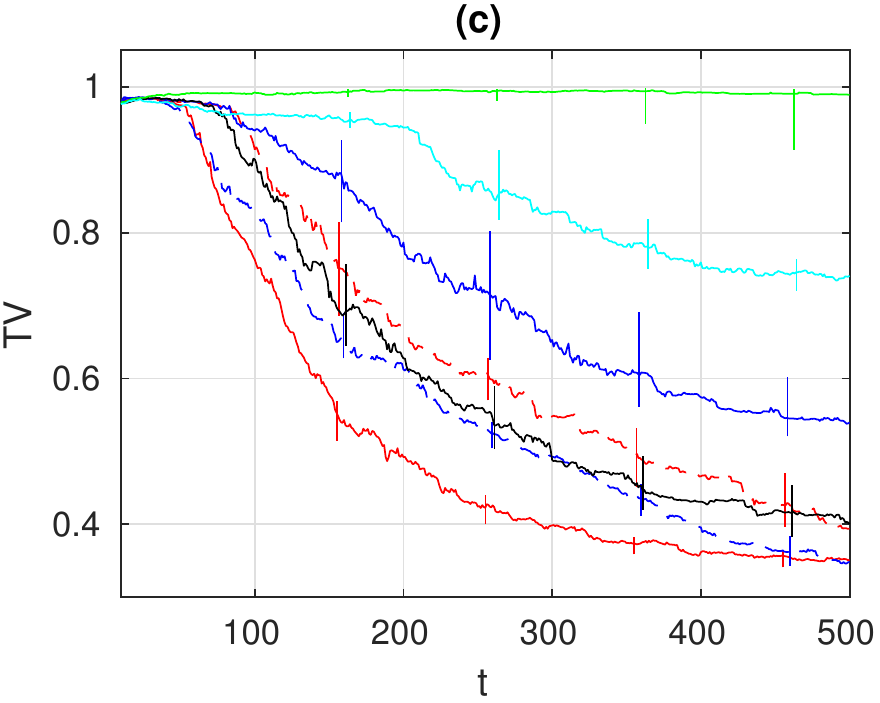}
\end{subfigure}
\caption
{Median of the mean absolute error in the (a) {\color{\revcol}ABC} posterior mean, and (b) {\color{\revcol}ABC} posterior variance {\color{\revcol}as compared to the true posterior} over 100 experiments in the Lotka-Volterra model. Panel (c) shows the TV distance between the estimated {\color{\revcol}ABC posterior} and the true posterior. 
} \label{fig:lv}
\end{figure*}


\subsubsection{Bacterial infections model}

{\color{\revcol} 
Finally, we show the promise of our method using a simulation model that describes transmission dynamics of bacterial infections in day care centers. 
The model has three parameters: an internal infection parameter $\beta \in [0,11]$, an external infection parameter $\Lambda \in [0,2]$ and a co-infection parameter $\theta \in [0,1]$.
Full details of the model and data are described in \citet{Numminen2013}. 
The true posterior is not available and thus an ABC posterior computed using PMC-ABC algorithm, which required over two million simulations, is used as the reference distribution \citep{Numminen2013}.
We use the same experimental setup and discrepancy as \citet{Gutmann2015}, who used the model to illustrate their approach. Specifically, the initial training data size is $t_0 = 20$ and the uniform prior is used. Adaptive MCMC is again used to sample from the model-based posterior estimates and from the probability density $\pi_q(\Btheta)$. For \expintvar{} we use $s=500$ importance samples.

Figure \ref{fig:bacterial} shows the results. 
Unlike in the other test cases, \expintvar{} and \randmaxvar{} tend to produce slightly wider credible intervals for the marginal ABC posterior distributions than the other methods. 
Similarly, \citet{Gutmann2015} obtained conservative estimates of these credible intervals with their stochastic variant of the LCB acquisition rule. 
To explain this, we investigated the GP modelling assumptions in more detail. 
Running a high number of additional bacterial model simulations indicates that the discrepancy is well approximated with a Gaussian in the modal area. On the other hand, the variance of the discrepancy, represented by the noise parameter $\sigma_n^2$ in the GP model, is not exactly constant as assumed in the GP but grows towards the tail areas. This explains why the \expintvar{} and \randmaxvar{} strategies, that tend to include more evaluations in the tails than the other methods, have a tendency to over-estimate the value of $\sigma_n^2$. This further causes slightly overestimated credible intervals for the marginal ABC posterior distribution, as illustrated by \citet{Jarvenpaa2016}. We can also see that the \expintvar{} strategy is more stable than the other methods in the sense that its results are more consistent over the $100$ different realisations of the initial training data sets and simulator outputs than those of the other methods. 
}

\begin{figure*}
\centering
\includegraphics[width=0.97\textwidth]{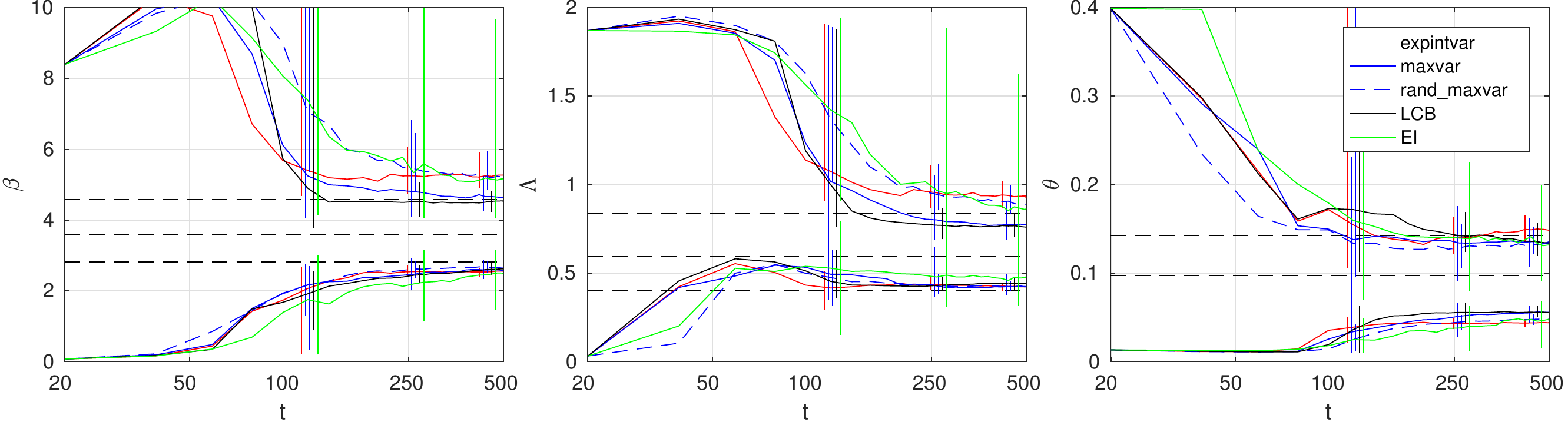}
\caption
{Comparison of the 95\% credible interval estimates in the bacterial model. The black dashed lines show {\color{\revcol} the corresponding credible interval estimates (computed using over 2 million model simulations with ABC-PMC algorithm)} by \citet{Numminen2013} and the vertical lines show the 75\% interval of the realisations over 100 experiments. The x-axis shows the iterations $t$ on the log-scale. 
} \label{fig:bacterial}
\end{figure*}

{\color{\revcol} 
The \maxvar{} and (deterministic) LCB produce credible interval estimates that are overall closest to the ABC posterior computed in \citet{Numminen2013}. However, the credible region for the parameter $\Lambda$ is often underestimated possibly due to excess exploitation producing too small variance estimates $\sigma_n^2$. 
Using input-dependent GP model as in \citet{Jarvenpaa2016} would likely improve the approximation quality but we do not investigate this possibility here. 
While the EI strategy appears to work well on average in this example, it actually has high variability and occasionally produces too narrow posterior estimates and thus performs poorly overall. 
Comparing our posterior approximations using both the proposed acquisition methods as well as the (deterministic) LCB strategy to those in \citet{Gutmann2015} shows that our estimates, both \expintvar{} and (deterministic) LCB, are more accurate than the experiment reported in their paper. 

In summary, despite some violations of the GP model assumptions, we were able to obtain posterior estimates that were very similar to those presented in \citet{Numminen2013} with only a fraction of simulations (500 vs 2,000,000), and without a need to use a computer cluster. We also showed that the proposed methods, especially \expintvar{}, work consistently over different simulation model realisations, which is important with any realistic model where extensive running times may prohibit proper assessment of stochastic variability.  
}

\section{Discussion} \label{sec:discussion}

In this section we offer guidelines for potential users and discuss some additional details of our algorithms. 
%
While the developed methods worked well, it may not be clear for an end user which method to use in practice. 
{\color{\revcol} First of all, if the likelihood can be evaluated, there is usually no reason to consider ABC. Furthermore, if this is not the case but the simulation is fast, e.g.~less than a second, standard ABC techniques may suffice. If the simulation is slow then the techniques in this paper become useful. }
Our technique with the \expintvar{} strategy has a sound Bayesian decision theoretic basis and it performed the best overall, producing consistent approximation quality in different scenarios. We thus recommend this strategy. 
%
However, \expintvar{} required up to 1 minute of computation time for selecting the next evaluation location in our 6d Gaussian test problem while this optimisation step took up to 4s for UCB and EI, and at most 8s for \maxvar{}. The corresponding time for \randmaxvar{} was 30s. In our 2d Gaussian case, the \expintvar{} strategy required at most 10s and all the other methods less than a second. These values are, however, descriptive since the computational time depends on various settings and the amount of training data. All these times are in any case negligible compared to the run time of many realistic simulation models which can be hours or days. In the supplementary material we {\color{\revcol} provide further analysis of computation time using Big O notation and we further} consider an alternative approach where a non-uniform acceptance threshold is used which allows for slightly faster computations.
However, in high dimensions, when $p \gtrsim 10$, we recommend \maxvar{} because we expect the estimated ABC posterior uncertainty to be inflated then anyway.

While not designed for ABC, the LCB criterion still worked surprisingly well overall and offers another reasonable choice in practice.
However, standard LCB is not suitable if the prior is informative and its accuracy deteriorated in the high-dimensional experiment. 
EI (and PI) performed poorly and we see little reason to use them unless the goal is to learn only the maximiser of the discrepancy. 
%
Furthermore, unlike the standard BO strategies such as LCB \citep[with the exception of][]{Shahriari2016}, the developed acquisition rules do not necessarily require a box-constrained domain. Namely, if the prior is a bounded and proper density (to ensure that the acquisition functions are bounded and the ABC posterior defines a valid pdf), the requirement of the bounded support can be relaxed.

In addition to the acquisition strategy, the posterior approximation quality also depends on the GP model and some other choices. 
For example, the proposed acquisition strategies and the final posterior estimate depend on the threshold. 
We either assumed this value to be known or used a heuristic approach and set the threshold to the $0.01$th quantile of the realised discrepancies. We also considered other choices but this approach worked well. In principle, the strategy for selecting the threshold could also vary during the iterations. 
While some ABC methods bypass selecting the threshold, they may not be applicable when the budget for simulations is very small. Our framework is also applicable for model-based ABC methods that do not require the threshold. 


We used the zero mean GP model in our experiments. While \citet{Wilkinson2014,Drovandi2015b,Gutmann2015} considered certain parametric mean functions which might help focusing the simulations on the modal area, our choice is a safe option. Namely, if there is a large region containing no simulations, the discrepancy tends to zero there. Thus, the uncertainty will be high in the region, attracting future simulations. 
Futhermore, even though in some studies \citep[e.g.][]{Snoek2012}, the Matern kernel has been empirically shown to perform slightly better than the squared exponential, we expect the discrepancy in many ABC modelling applications to be smooth and, consequently, used the squared exponential kernel. 


To demonstrate the framework, we chose to model the discrepancy with a GP. However, this approach may not be optimal if the Gaussianity assumption is violated \citep{Gutmann2015,Jarvenpaa2016}. 
Non-Gaussian measurement models can be used but the acquisition criteria in Equation (\ref{eq:post_var}) or (\ref{eq:expintvar}) may become costly to evaluate. 
One could also model the log-likelihood directly with a GP and select the evaluations at the maximiser of the variance of the likelihood function \citep{Kandasamy2015}, which is similar to our \maxvar{} criterion. One could also model the individual summaries with independent GPs as in \cite{Jabot2014,Meeds2014}. In both of these cases the evaluation locations could be chosen based on the ideas in Section \ref{sec:problem}.

An alternative to the proposed stochastic acquisition rule is to sample new evaluation locations from the current ABC posterior estimate. This approach seems to work well in some scenarios but no systematic comparison was done. 
However, the posterior estimate could get stuck to a poor region due to an ``unlucky'' discrepancy realisation, after which new evaluations would be focused on this seemingly good region and the method has little chance to escape from the local optimum. 

While our approach is designed for fitting costly simulation models, we note that it can be useful even when the simulation model is relatively cheap to run. For example, for a developer of a simulation model, it may be useful to first obtain rough estimates for the model parameters before using costly computations for final and accurate results. 
Our derivations are also applicable for estimating the tail probabilities of Gaussian processes over some parameter domain. An approach similar to ours has also been applied to the problem of estimating an excursion set by \citet{Chevalier2014}. However, the objective of their work is to identify the set of points that are below a fixed threshold instead of learning the corresponding tail probability under GP surrogate model assumptions.

\section{Conclusions} \label{sec:conclusions}

We considered the challenging problem of performing Bayesian inference when the likelihood function cannot be evaluated and simulating data from the statistical model is costly. 
We proposed to use another instance of Bayesian inference to quantify the uncertainty in the approximate posterior due to the limited budget of simulations and to design the simulations to minimise the expected uncertainty in the posterior approximation. Such computations can be costly themselves but we chose a loss function that measures such uncertainty and allowed developing a tractable and practical algorithm for selecting the next evaluation location to run the simulation model. Notably, compared to many realistic simulation models, the run time of which can be hours or days, the computational overhead introduced by our approach is negligible. 
%
Experiments demonstrated that the proposed method performs better than or similarly to the commonly used Bayesian optimisation strategies and other, more heuristic approaches obtained as a by-product of our derivations. Our approach also takes prior density into account, does not require box-constrained parameter spaces and has a sound decision-theoretic basis that extends to other ABC surrogate modelling scenarios beyond those considered in this article. 

%
%
As future work, other surrogate models and principled approaches for selecting the threshold could be investigated. 
%
We here focused on single acquisitions but our approach in principle extends to batch acquisitions as well. This enables parallelised inference, which is particularly useful when computationally very costly simulation models need to be fitted.

\subsection*{Acknowledgements}

This work was funded by the Academy of Finland (grants no. 286607 and 294015 to PM). We acknowledge the computational resources provided by Aalto Science-IT project.

{
\subsection*{References}
\renewcommand{\section}[2]{}
\bibliography{abc_acquisition_new_v2_arxiv.bib}
\bibliographystyle{plainnat}
}

\appendix

\section{Proofs} \label{app:proofs}

We provide the derivations of Lemma \ref{lemma:p_mean_var} and Proposition \ref{prop:expintvar} below.

\begin{proof}[Proof of Lemma \ref{lemma:p_mean_var}]
Using the law of the unconscious statistician, we can write 
\begin{equation}
\mean(\pacc(\Btheta)) = \int_{-\infty}^{\infty} \Phi\left(\frac{\epsilon - f}{\sigma_n}\right) \Normal(f\cond m_{\idxt}(\Btheta), v_{\idxt}^2(\Btheta)) \ud f .
\end{equation}
%
Now, using the fact $\Phi(x) = 1 - \Phi(-x)$, a standard result for Gaussian moments derived in \citet[p.~74]{Rasmussen2006} and $\mean(\pi(\Btheta) \pacc(\Btheta)) = \pi(\Btheta) \mean(\pacc(\Btheta))$, (which holds because the prior density at $\Btheta$ is a scalar) one obtains Equation (\ref{eq:post_mean}). 

A formula for the variance of $\tildepiabc(\Btheta)$ can be obtained similarly. First we see that 
\begin{align}
\Var(\pacc(\Btheta)) &= \mean(\pacc(\Btheta)^2) - [\mean(\pacc(\Btheta))]^2 \\
&= \int_{-\infty}^{\infty} \Phi^2\left(\frac{\epsilon - f}{\sigma_n}\right) \Normal(f\cond m_{\idxt}(\Btheta), v_{\idxt}^2(\Btheta)) \ud f 
- [\mean(\pacc(\Btheta))]^2. \label{eq:var_first_eq}
\end{align}
The first term of Equation (\ref{eq:var_first_eq}) can be further written as 
\begin{align}
\begin{aligned}
&\int_{-\infty}^{\infty} \int_{-\infty}^{\epsilon} \Normal(z_1 \cond f, \sigma_n^2) \ud z_1
\int_{-\infty}^{\epsilon} \Normal(z_2 \cond f, \sigma_n^2) \ud z_2 
\Normal(f \cond m_{\idxt}(\Btheta), v_{\idxt}^2(\Btheta)) \ud f \\
&= \int_{-\infty}^{\epsilon} \int_{-\infty}^{\epsilon} \int_{-\infty}^{\infty} 
\Normal(f \cond z_1, \sigma_n^2) \Normal(f \cond z_2, \sigma_n^2) \Normal(f \cond m_{\idxt}(\Btheta), v_{\idxt}^2(\Btheta)) \ud f \ud z_1 \ud z_2 
\label{eq:tricky_formula}
\end{aligned}
\end{align}
where we have used Fubini-Tonelli theorem to change the order of integration. The integrand can be now written as an unnormalised Gaussian pdf for $f$ and, after integrating over $f$ and some further calculations, the resulting formula can be recognised as 
\begin{align}
\Phi_{2}(\epsilon\onevector \cond m_{\idxt}(\Btheta)\onevector, V_{\idxt}(\Btheta)),
\label{eq:var}
\end{align}
where $\onevector = [1, 1]^T$ and the function $\Phi_2$ denotes the bivariate Gaussian cdf with mean $m_{\idxt}(\Btheta)\onevector$ and  covariance matrix
\begin{equation}
V_{\idxt}(\Btheta) 
= \begin{bmatrix} \sigma_n^2 + v_{\idxt}^2(\Btheta) & v_{\idxt}^2(\Btheta) \\
v_{\idxt}^2(\Btheta) & \sigma_n^2 + v_{\idxt}^2(\Btheta) \end{bmatrix},
\end{equation}
which is clearly symmetric and positive definite since $\sigma_n^2 > 0$.

Denoting the correlation coefficient $\rho(\Btheta) = v^2_{\idxt}(\Btheta)/(\sigma_n^2+v^2_{\idxt}(\Btheta))$, we see that Equation (\ref{eq:var}) can be further written as 
%
\begin{align}
\begin{aligned}
%
&\Phi_{2}\left(\frac{\epsilon - m_{\idxt}(\Btheta)}{\sqrt{\sigma_n^2 + v_{\idxt}^2(\Btheta)}}\onevector \mcond \Bzeros, \begin{bmatrix} 1 & \rho(\Btheta) \\ \rho(\Btheta) & 1 \end{bmatrix} \right) \\
%
&= \Phi\left(\frac{\epsilon-m_{\idxt}(\Btheta)}{\sqrt{\sigma_n^2 + v_{\idxt}^2(\Btheta)}}\right) 
- 2T\left( \frac{\epsilon - m_{\idxt}(\Btheta)}{\sqrt{\sigma_n^2 + v_{\idxt}^2(\Btheta)}}, \frac{1-\rho(\Btheta)}{\sqrt{1-\rho^2(\Btheta)}} \right) 
%
%
\end{aligned}
\end{align}
where $\Bzeros = [0,0]^T$. The 
equality follows from the connection between bivariate Gaussian cdf and Owen's t-function, see \citet{Owen1956} for this fact and its proof. 
Also, simple calculations show that
\begin{equation}
\frac{1-\rho(\Btheta)}{\sqrt{1-\rho^2(\Btheta)}} 
%
%
= \frac{\sigma_n}{\sqrt{\sigma_n^2 + 2v_{\idxt}^2(\Btheta)}}.
\end{equation}
The rest of the result now follows from the derivations above, Equation (\ref{eq:post_mean}) and the fact $\Var(\pi(\Btheta) \pacc(\Btheta)) = \pi^2(\Btheta) \Var(\pacc(\Btheta))$. 
\end{proof}

\begin{proof}[Proof of Proposition \ref{prop:expintvar}]
%
We first derive the probability densities for the GP mean and covariance function after a future observation is obtained. These quantities are random variables because the new discrepancy realisation $\Delta^{*}$ is unknown.
Given the training data $D_{\idxt} = \{(\Delta_i,\Btheta_i)\}_{i=1}^t$, the mean function $m_{1:t+1}(\Btheta)$ can be written as 
\begin{align}
m_{1:t+1}(\Btheta) &= [k_{\Btheta,\Btheta_{\idxt}}, k_{\Btheta,\Btheta^{*}}] \begin{bmatrix} K(\Btheta_{\idxt}) & k(\Btheta_{\idxt},\Btheta^{*}) \\ k(\Btheta^{*},\Btheta_{\idxt}) & K(\Btheta^{*}) \end{bmatrix}^{-1} \begin{bmatrix} \Delta_{\idxt} \\ \Delta^{*} \end{bmatrix} \\ 
%
%
&= k(\Btheta,\Btheta_{\idxt})K^{-1}(\Btheta_{\idxt})\Delta_{\idxt} + [k(\Btheta,\Btheta^{*}) - k(\Btheta,\Btheta_{\idxt})K^{-1}(\Btheta_{\idxt})k(\Btheta_{\idxt},\Btheta^{*})] 
\nonumber \\
&\qquad \cdot [K(\Btheta^{*}) - k(\Btheta^{*},\Btheta_{\idxt})K^{-1}(\Btheta_{\idxt})k(\Btheta_{\idxt},\Btheta^{*})]^{-1}
[\Delta^{*} - k(\Btheta^{*},\Btheta_{\idxt})K^{-1}(\Btheta_{\idxt})\Delta_{\idxt}] \nonumber \\
%
%
&= m_{\idxt}(\Btheta) + \cov_{\idxt}(\Btheta,\Btheta^{*})(v^2_{\idxt}(\Btheta^{*}) + \sigma_n^2)^{-1}(\Delta^{*} - m_{\idxt}(\Btheta^{*})),
\end{align}
where we have used a well-known formula for blockwise inversion. 
According to the current GP model, the unknown future discrepancy $\Delta^{*}$ follows a Gaussian density i.e.~$\Delta^{*}\cond\Btheta^{*},D_{\idxt} \sim \Normal(m_{\idxt}(\Btheta^{*}), v^2_{\idxt}(\Btheta^{*}) + \sigma_n^2)$ so that
\begin{align}
m_{\idxtp}(\Btheta) &\sim \Normal(m_{\idxt}(\Btheta), \Deltav_{\idxt}^2(\Btheta,\Btheta^{*})), \label{eq:gp_upd_mean_sim}
%
\end{align}
where we have denoted $\Deltav_{\idxt}^2(\Btheta,\Btheta^{*}) = {\cov_{\idxt}^2(\Btheta,\Btheta^{*})}/{(\sigma_n^2 + v_{\idxt}^2(\Btheta^{*}))}$.

Similar computations as for the mean function show that 
\begin{equation}
v^2_{\idxtp}(\Btheta) = v^2_{\idxt}(\Btheta) - \Deltav_{\idxt}^2(\Btheta,\Btheta^{*}). \label{eq:gp_upd_cov}
\end{equation}
This formula further shows that the change in the GP variance is deterministic and depends only on the chosen evaluation location $\Btheta^{*}$.

Changing the order of integration using Tonelli theorem we obtain 
\begin{align}
L_{\idxt}(\Btheta^{*}) &= \mean_{\Delta^{*}\cond \Btheta^{*}, D_{\idxt}} \int_{\Theta} \pi^2(\Btheta) \Var(\pacc(\Btheta) \cond \Delta^{*},\Btheta^{*},D_{\idxt}) \ud \Btheta \\
&= \int_{-\infty}^{\infty} \int_{\Theta} \pi^2(\Btheta) \Var(\pacc(\Btheta) \cond \Delta^{*},\Btheta^{*},D_{\idxt}) \ud \Btheta \Normal(\Delta^{*} \cond m_{\idxt}(\Btheta^{*}), \sigma_n^2 + v^2_{\idxt}(\Btheta^{*})) \ud \Delta^{*} \\
&= \int_{\Theta} \pi^2(\Btheta) \underbrace{\int_{-\infty}^{\infty} \Var(\pacc(\Btheta) \cond \Delta^{*},\Btheta^{*},D_{\idxt}) \Normal(\Delta^{*} \cond m_{\idxt}(\Btheta^{*}), \sigma_n^2 + v^2_{\idxt}(\Btheta^{*})) \ud \Delta^{*}}_{ =: g_{\idxtp}^2(\Btheta,\Btheta^{*})} \ud \Btheta \\
&= \int_{\Theta} \pi^2(\Btheta) g_{\idxtp}^2(\Btheta,\Btheta^{*}) \ud \Btheta,
\end{align}
%
%
To simplify the notation in the following formulas, we define $m_0 := m_{\idxt}(\Btheta)$, $m_1 := m_{\idxtp}(\Btheta)$ and similarly for the variance terms $v^2_0$ and $v^2_1$. We also define $\Deltav_0^2(\Btheta^{*}) := \Deltav^2_{\idxt}(\Btheta,\Btheta^{*})$. Using these conventions and Equations (\ref{eq:gp_upd_mean_sim}) and (\ref{eq:gp_upd_cov}) we see that
\begin{align}
%
&g_{\idxtp}^2(\Btheta,\Btheta^{*}) \nonumber \\ &\myquad= \int_{0}^{\infty}\int_{-\infty}^{\infty} \Var(\pacc(\Btheta,m_1,v^2_1)) \Normal(m_1 \cond m_0, \Deltav_0^2(\Btheta^{*})) 
\delta(v^2_1 - v^2_0 + \Deltav_0^2(\Btheta^{*})) \ud m_1 \ud v^2_1 \\
%
&\myquad= \int_{-\infty}^{\infty} \Phi\left(\frac{\epsilon - m_1}{\sqrt{\sigma_n^2 + v_1^2 - \Deltav_0^2(\Btheta^{*})}}\right) - \Phi^2\left(\frac{\epsilon - m_1}{\sqrt{\sigma_n^2 + v_1^2 - \Deltav_0^2(\Btheta^{*})}}\right) \Normal(m_1 \cond m_0, \Deltav_0^2(\Btheta^{*})) \ud m_1 \label{eq:easy_terms} \\
%
&\myquad\myquad- \frac{1}{\pi} \int_{-\infty}^{\infty} \int_{0}^{\frac{\sigma_n}{\sqrt{\sigma_n^2 + 2v_1^2 - 2\Deltav_0^2(\Btheta^{*})}}} \frac{\e^{-\frac{1}{2}\left(\frac{\epsilon - m_1}{\sqrt{\sigma_n^2 + v_1^2- \Deltav_0^2(\Btheta^{*})}}\right)^2(1 + x^2)}}{1 + x^2} \ud x \Normal(m_1 \cond m_0, \Deltav_0^2(\Btheta^{*})) \ud m_1 \label{eq:owen_integral}
\end{align}
%

The integral of Equation (\ref{eq:easy_terms}) can be computed using the equations in the proof of Lemma \ref{lemma:p_mean_var}
and after some straightforward computations one obtains 
\begin{equation}
2T\left(\frac{\epsilon - m_0}{\sqrt{\sigma_n^2 + v_0^2}}, 
\sqrt\frac{\sigma_n^2 + v_0^2(\Btheta) - \Deltav_0^2(\Btheta^{*})}{\sigma_n^2 + v_0^2(\Btheta) + \Deltav_0^2(\Btheta^{*})}\right).
\end{equation}
%

To simplify Equation (\ref{eq:owen_integral}), we use again the Fubini-Tonelli theorem to change the order of integration and some straightforward (but tedious) manipulations to obtain 
\begin{align}
&-\frac{1}{\pi} \int_{0}^{\frac{\sigma_n}{\sqrt{\sigma_n^2 + 2v_1^2 - 2\Deltav_0^2(\Btheta^{*})}}} 
\int_{-\infty}^{\infty} \frac{\e^{-\frac{1}{2}\frac{(\epsilon - m_1)^2{(1 + x^2)}}{\sigma_n^2 + v_1^2 - \Deltav^2_0(\Btheta^{*})}}}{1 + x^2} \Normal(m_1 \cond m_0, \Deltav_0^2(\Btheta^{*})) \ud m_1 \ud x \\
&= -\frac{1}{\pi} \int_{0}^{\frac{\sigma_n}{\sqrt{\sigma_n^2 + 2v_1^2 - 2\Deltav_0^2(\Btheta^{*})}}} \int_{-\infty}^{\infty} \frac{\sqrt{2\pi}c(x)}{1+x^2} \Normal(m_1 \cond \epsilon, c^2(x)) \Normal(m_1 \cond m_0, \Deltav_0^2(\Btheta^{*})) \ud m_1 \ud x \\
&= -\frac{1}{\pi} \int_{0}^{\frac{\sigma_n}{\sqrt{\sigma_n^2 + 2v_1^2 - 2\Deltav_0^2(\Btheta^{*})}}} \frac{1}{1+x^2} \sqrt{\frac{\sigma_n^2 + v_0^2 - \Deltav_0^2(\Btheta^{*})}{\sigma_n^2 + v_0^2 + x^2 \Deltav_0^2(\Btheta^{*})}} \e^{-\frac{(\epsilon - m_1)^2(1+x^2)}{\sigma_n^2 + v_0^2 + x^2 \Deltav_0^2(\Btheta^{*})}} \ud x 
\label{eq:complicated_eq}
\end{align}
where we have defined $c(x) := (\sigma_n^2 + v_1^2 - \Deltav^2_0(\Btheta^{*}))/(1+x^2)$ and used again the well-known product rule for Gaussian pdfs.

Next we define the transformation 
$\psi(z) := z\sqrt{\sigma_n^2 + v_0^2}/\sqrt{\sigma_n^2 + v_0^2 - \Deltav^2_0(\Btheta^{*})(1+z^2)}$. 
Some analysis shows that 
$\psi'(z) = \sqrt{\sigma_n^2 + v_0^2}\sqrt{\sigma_n^2 + v_0^2 - \Deltav^2_0(\Btheta^{*})}/(\sigma_n^2 + v_0^2 - \Deltav^2_0(\Btheta^{*})(1+z^2))^{3/2} > 0 $
so that $\psi$ is strictly increasing function, and that it maps the interval $[0,{\sigma_n}/\sqrt{\sigma_n^2 + 2v_0^2}]$ to $[0,{\sigma_n}/{\sqrt{\sigma_n^2 + 2v_1^2 - 2\Deltav_0^2(\Btheta^{*})}}]$. 
Substituting $x = \psi(z)$ to the integral in Equation (\ref{eq:complicated_eq}) and some straightforward computations show that Equation (\ref{eq:complicated_eq}) simplifies to 
\begin{align}
-\frac{1}{\pi}\int_0^{\frac{\sigma_n}{\sqrt{\sigma_n^2 + 2v_0^2}}} \frac{1}{1+z^2} \e^{-\frac{(\epsilon - m_0)^2(1+z^2)}{2(\sigma_n^2 + v_0^2)}} \ud z 
= - 2T\left( \frac{\epsilon - m_0}{\sqrt{\sigma_n^2 + v_0^2}}, \frac{\sigma_n}{\sqrt{\sigma_n^2 + 2v_0^2}} \right).
\end{align}
The final result (\ref{eq:expintvar}) now follows from the equations above. 
\end{proof}

\section{Additional derivations and gradients}


\subsection{Additional derivations} \label{app:derivations}

We start by deriving the cdf for the random variable $\pacc(\Btheta)$ when the uncertainty in the GP hyperparameters $\Bphi$ is taken into account. The corresponding results for $\tildepiabc(\Btheta)$ follow by suitable scaling with the prior pdf $\pi(\Btheta)$. The cdf of $\pacc(\Btheta)$ evaluated at $z\in(0,1)$ is
\begin{align}
\begin{split}
F_{\pacc(\Btheta)}(z) &= \prob(\pacc(\Btheta) \leq z) 
= \int \prob(\pacc(\Btheta) \leq z \cond \Bphi) \pi(\Bphi) \ud \Bphi 
= \int \prob\left(\Phi\left(\frac{\epsilon - f(\Btheta)}{\sigma_n}\right) \leq z \mcond \Bphi\right) \pi(\Bphi) \ud \Bphi \\
&= \int \prob\left(f(\Btheta)) \geq \epsilon - \sigma_n\Phi^{-1}(z) \mcond \Bphi\right) \pi(\Bphi) \ud \Bphi
= \int \Phi\left( \frac{\sigma_n\Phi^{-1}(z) + m_{\idxt}(\Btheta\cond\Bphi) - \epsilon}{v_{\idxt}(\Btheta\cond\Bphi)} \right) \pi(\Bphi) \ud \Bphi,
\end{split}
\end{align}
it is zero if $z \leq 0$, and one if $z \geq 1$. In above, the density $\pi(\Bphi)$ describes our knowledge about the GP hyperparameters $\Bphi$ given the training data $D_{\idxt}$ (conditioning on data is ignored to simplify notation). The integral in the equations above is taken over the domain of the GP hyperparameters $\Bphi$. The integral can be approximated using e.g.~CCD as discussed in the main text. 
If the hyperparameters are fixed and $\pi_{\Bphi}(\Bphi)$ is replaced with a point mass, one obtains Equation (\ref{eq:p_cdf}).

A formula for the pdf can be obtained by differentiating the cdf. We first realise that
%
$z = \Phi(\Phi^{-1}(z))$ which implies $1 = \frac{\ud z}{\ud z} = \frac{\ud}{\ud z} \Phi(\Phi^{-1}(z)) = \Phi'(\Phi^{-1}(z))(\Phi^{-1})'(z)$ 
%
for $z\in(0,1)$. This fact is used to further show that  
\begin{align}
(\Phi^{-1})'(z) = \frac{1}{\Phi'(\Phi^{-1}(z))} = \frac{1}{\Normal(\Phi^{-1}(z)\cond 0,1)} 
= \sqrt{2\pi}\e^{(\Phi^{-1}(z))^2/2}. \label{eq:dinvphi}
\end{align}
Using the Equation (\ref{eq:dinvphi}) allows to compute 
\begin{align}
\pi_{\pacc(\Btheta) \cond \Bphi}(z) &= \frac{\partial}{\partial z} F_{\pacc(\Btheta)\cond \Bphi}(z) 
= \frac{1}{\sqrt{2\pi}} \e^{-\frac{(\sigma_n\Phi^{-1}(z) + m_{\idxt}(\Btheta\cond\Bphi) - \epsilon)^2}{2v_{\idxt}^2(\Btheta\cond\Bphi)}} \frac{\partial}{\partial z} \frac{\sigma_n\Phi^{-1}(z) + m_{\idxt}(\Btheta\cond\Bphi) - \epsilon}{v_{\idxt}(\Btheta\cond\Bphi)} \\
&= \frac{\sigma_n}{v_{\idxt}(\Btheta\cond\Bphi)} \e^{\frac{(\Phi^{-1}(z))^2}{2} - \frac{(\sigma_n\Phi^{-1}(z) + m_{\idxt}(\Btheta\cond\Bphi) - \epsilon)^2}{2v_{\idxt}^2(\Btheta\cond\Bphi)}} \\
&= \begin{cases}
\frac{\sigma_n}{v_{\idxt}(\Btheta\cond\Bphi)} \e^{\frac{(\epsilon - m_{\idxt}(\Btheta\cond\Bphi))^2}{2(\sigma_n^2 - v_{\idxt}^2(\Btheta\cond\Bphi))}} \e^{ -\frac{\sigma_n^2 - v_{\idxt}^2(\Btheta\cond\Bphi)}{2v_{\idxt}^2(\Btheta\cond\Bphi)}\left(\Phi^{-1}(z) - \frac{(\epsilon - m_{\idxt}(\Btheta\cond\Bphi))\sigma_n}{\sigma_n^2 - v_{\idxt}^2(\Btheta\cond\Bphi)}\right)^2}, 
& \text{if } \sigma_n \neq v_{\idxt}(\Btheta\cond\Bphi), \\
\e^{-\frac{(\epsilon - m_{\idxt}(\Btheta\cond\Bphi))^2}{2v_{\idxt}^2(\Btheta\cond\Bphi)}} \e^{\frac{\epsilon - m_{\idxt}(\Btheta\cond\Bphi)}{v_{\idxt}(\Btheta\cond\Bphi)}\Phi^{-1}(z)}, 
& \text{if } \sigma_n = v_{\idxt}(\Btheta\cond\Bphi), 
\end{cases} 
\label{eq:pdf_final}
\end{align}
for $z\in(0,1)$ and it is zero elsewhere. Finally, the pdf is obtained by marginalising the GP hyperparameters, that is 
\begin{align}
\pi_{\pacc(\Btheta)}(z) = \int \pi_{\pacc(\Btheta) \cond \Bphi}(z) \pi_{\Bphi}(\Bphi) \ud \Bphi. 
\label{eq:pdf_final2}
\end{align}

The mean and variance for $\tildepiabc(\Btheta)$ were presented in Section \ref{subsec:uncertainty} and the corresponding results for $\pacc(\Btheta)$ follow by setting $\pi(\Btheta) = 1$. 
The quantiles can be computed as in Equation (\ref{eq:quantile}). If the uncertainty in the GP hyperparameters is taken into account, then numerical root finding such as bisection search is required for inverting the cdf. 
%

Inspecting the pdf given by Equation (\ref{eq:pdf_final}) shows that if $\sigma_n > v_{\idxt}(\Btheta\cond\Bphi)$, then the mode of $\pacc(\Btheta)\cond\Bphi$ is at $z = \Phi((\epsilon - m_{\idxt}(\Btheta\cond\Bphi))/(\sigma_n - v_{\idxt}^2(\Btheta\cond\Bphi)/\sigma_n))$. Unsurprisingly, if $m_{\idxt}(\Btheta\cond\Bphi)$ is large enough, then there is a mode near $z=0$. 
However, if $\sigma_n = v_{\idxt}(\Btheta\cond\Bphi)$ and $m_{\idxt}(\Btheta\cond\Bphi) > \epsilon$, then the pdf goes to infinity as $z \rightarrow 0^{+}$. Interestingly, if $\sigma_n < v_{\idxt}(\Btheta\cond\Bphi)$, then the pdf goes to infinity both as $z \rightarrow 0^{+}$ and $z \rightarrow 1^{-}$. 

%
%
%
%


\subsection{Gradient of the expected integrated variance acquisition function} \label{sec:expintvar_gradient}

We outline the derivation for the gradient of the expected integrated variance acquisition function (Equation (\ref{eq:expintvar})) with respect to the candidate evaluation location $\Btheta^{*}$. We consider only the case where either a point estimate or a fixed value is used for the GP hyperparameters $\Bphi$. First we define 
\begin{equation}
c(\Btheta,\Btheta^*) 
:= \sqrt\frac{\sigma_n^2 + v_{\idxt}^2(\Btheta) - \Deltav_{\idxt}^2(\Btheta,\Btheta^{*})}{\sigma_n^2 + v_{\idxt}^2(\Btheta) + \Deltav_{\idxt}^2(\Btheta,\Btheta^{*})}.
\end{equation}
Because the second term in Equation (\ref{eq:expintvar}) is constant with respect to $\Btheta^{*}$, we obtain
\begin{align}
\frac{\partial}{\partial \Btheta^{*}} L_{\idxt}(\Btheta^{*}) 
&= 2 \; \frac{\partial}{\partial \Btheta^{*}} \int_{\Theta} \pi^2(\Btheta) T\left(\frac{\epsilon - m_{\idxt}(\Btheta)}{\sqrt{\sigma_n^2 + v_{\idxt}^2(\Btheta)}},
c(\Btheta,\Btheta^*) \right) \ud \Btheta \\
%
&= \frac{1}{\pi} \int_{\Theta} \pi^2(\Btheta) \frac{\partial}{\partial \Btheta^{*}} \int_0^{c(\Btheta,\Btheta^*)} \frac{\e^{-\frac{(\epsilon - m_{1:t}(\Btheta))^2}{2(\sigma_n^2 + v^2_{1:t}(\Btheta))}(1+x^2)}}{1+x^2} \ud x \ud \Btheta \\
%
&= \frac{1}{\pi} \int_{\Theta} \pi^2(\Btheta) \frac{\e^{-\frac{(\epsilon - m_{1:t}(\Btheta))^2}{2(\sigma_n^2 + v^2_{1:t}(\Btheta))}(1+c^2(\Btheta,\Btheta^*))}}{1+c^2(\Btheta,\Btheta^*)} \frac{\partial}{\partial \Btheta^{*}} c(\Btheta,\Btheta^*) \ud \Btheta,
\end{align}
where we have used the Leibnitz integration rule twice and where the integrations are applied elementwise. 
Differentiating $c(\Btheta,\Btheta^*)$ and some further computations produce %
%
%
\begin{equation}
\frac{\partial}{\partial \Btheta^{*}} L_{\idxt}(\Btheta^{*}) 
= -\frac{1}{2\pi} \int_{\Theta} \frac{\pi^2(\Btheta) \e^{-\frac{(\epsilon - m_{\idxt}(\Btheta))^2}{\sigma_n^2 + v^2_{\idxt}(\Btheta) + \Deltav^2_{\idxt}(\Btheta,\Btheta^{*})}}}{\sqrt{\sigma_n^2 + v^2_{\idxt}(\Btheta) + \Deltav^2_{\idxt}(\Btheta,\Btheta^{*})}\sqrt{\sigma_n^2 + v^2_{\idxt}(\Btheta) - \Deltav^2_{\idxt}(\Btheta,\Btheta^{*})}}
\frac{\partial}{\partial \Btheta^{*}} \Deltav^2_{\idxt}(\Btheta,\Btheta^{*})
\ud \Btheta, 
\label{eq:expintvar_grad_main}
\end{equation}
where 
\begin{align}
\frac{\partial}{\partial \Btheta^{*}} \Deltav^2_{\idxt}(\Btheta,\Btheta^{*}) 
&= \frac{2\cov_{\idxt}(\Btheta,\Btheta^{*})}{\sigma_n^2 + v_{\idxt}^2(\Btheta^{*})} \frac{\partial}{\partial \Btheta^{*}} \cov_{\idxt}(\Btheta,\Btheta^{*})
- \frac{\cov^2_{\idxt}(\Btheta,\Btheta^{*})}{(\sigma_n^2 + v_{\idxt}^2(\Btheta^{*}))^2} \frac{\partial}{\partial \Btheta^{*}} v^2_{\idxt}(\Btheta^{*}), \label{eq:grad_dgpvar} \\
\frac{\partial}{\partial \Btheta^{*}} \cov_{\idxt}(\Btheta,\Btheta^{*}) 
&= \frac{\partial}{\partial \Btheta^{*}} k(\Btheta,\Btheta^{*}) - k(\Btheta,\Btheta_{\idxt})K^{-1}(\Btheta_{\idxt},\Btheta_{\idxt}) \frac{\partial}{\partial \Btheta^{*}} k(\Btheta_{\idxt},\Btheta^{*}). 
\label{eq:grad_gpcov}
\end{align}
The integral in Equation (\ref{eq:expintvar_grad_main}) can be approximated similarly as discussed in Section \ref{subsec:acquisition_rule}.

\subsection{Gradient of the maxvar acquisition function} \label{app:derivations2}

We compute the gradient of the maxvar acquisition function with respect to the parameter vector $\Btheta$. We take into account the uncertainty in the GP hyperparameters but if a point estimate is used instead, the formulae can be simplified by ignoring the summations, setting $\omega^i = \omega^1 = 1$ and replacing $\Bphi^i$ with the point estimate. First we denote
\begin{align}
I_1(\Btheta) &= \sum_{i} \omega^i \Phi(a(\Btheta,\Bphi^i)) - \left[\sum_{i} \omega^i \, \Phi(a(\Btheta,\Bphi^i)) \right]^{2}, \\
I_2(\Btheta) &= \frac{1}{\pi} \sum_{i} \omega^i \int_0^{b(\Btheta,\Bphi^i)} \frac{\e^{-\frac{1}{2}a^2(\Btheta,\Bphi^i)(1+x^2)}}{1+x^2} \ud x, 
\end{align}
where 
\begin{align}
a(\Btheta,\Bphi^i) = \frac{\epsilon - m_{\idxt}(\Btheta\cond \Bphi^i)}{\sqrt{(\sigma_n^i)^2 + v_{\idxt}^2(\Btheta\cond \Bphi^i)}}, 
\quad b(\Btheta,\Bphi^i) = \frac{\sigma_n^i}{\sqrt{(\sigma_n^i)^2 + 2v_{\idxt}^2(\Btheta\cond \Bphi^i)}},
\label{eq:ab}
\end{align}
so that
\begin{align}
\pi^2(\Btheta)\Var(\pacc(\Btheta)) 
&\approx \pi^2(\Btheta)(I_1(\Btheta) - I_2(\Btheta)). \label{eq:pi2var}
\end{align}
Differentiating Equation (\ref{eq:pi2var}) with respect to the parameter vector $\Btheta$ yields
\begin{align}
\frac{\partial}{\partial \Btheta}\pi^2(\Btheta)\Var(\pacc(\Btheta)) 
&\approx 2\pi(\Btheta)(I_1(\Btheta) - I_2(\Btheta))\frac{\partial \pi(\Btheta)}{\partial \Btheta} + \pi^2(\Btheta)\left(\frac{\partial I_1(\Btheta)}{\partial \Btheta} - \frac{\partial I_2(\Btheta)}{\partial \Btheta}\right). 
\end{align}
Computing the derivatives of $I_1$ produces
\begin{align}
\frac{\partial I_1(\Btheta)}{\partial \Btheta} 
&= \sum_{i} \omega^i \frac{\partial}{\partial \Btheta} \Phi(a(\Btheta,\Bphi^i)) - 2 \left( \sum_{i} \omega^i \Phi(a(\Btheta,\Bphi^i)) \right) \left( \sum_{i} \omega^i \frac{\partial}{\partial \Btheta} \Phi(a(\Btheta,\Bphi^i)) \right) \nonumber \\
&= \left( 1 - 2\sum_{i} \omega^i \Phi(a(\Btheta,\Bphi^i)) \right) \sum_{i} \omega^i \frac{\partial}{\partial \Btheta} \Phi(a(\Btheta,\Bphi^i)) \nonumber \\
&= \left( 1 - 2\sum_{i} \omega^i \Phi(a(\Btheta,\Bphi^i)) \right) \sum_{i} \frac{\omega^i \e^{-\frac{1}{2}a^2(\Btheta,\Bphi^i)}}{\sqrt{2\pi}} \frac{\partial a(\Btheta,\Bphi^i)}{\partial \Btheta}.
\end{align}

Using the Leibniz integration rule, the gradient of $I_2$ can be written as 
\begin{align}
\frac{\partial I_2(\Btheta)}{\partial \Btheta} &= \frac{1}{\pi} \sum_{i} \omega^i \left( \frac{\e^{-\frac{1}{2}a^2(\Btheta,\Bphi^i)(1+b^2(\Btheta,\Bphi^i))}}{1+b^2(\Btheta,\Bphi^i)}\frac{\partial b(\Btheta,\Bphi^i)}{\partial \Btheta} 
+ \int_{0}^{b(\Btheta,\Bphi^i)} \frac{1}{1+x^2} \frac{\partial}{\partial \Btheta} \e^{-\frac{1}{2}a^2(\Btheta,\Bphi^i)(1+x^2)} \ud x \right), 
\label{eq:I2_first} 
%
\end{align}
where the integration is applied elementwise. 
The second term in Equation (\ref{eq:I2_first}) can be further simplified as 
\begin{align}
&\int_{0}^{b(\Btheta,\Bphi^i)} \frac{1}{1+x^2} \frac{\partial}{\partial \Btheta} \e^{-\frac{1}{2}a^2(\Btheta,\Bphi^i)(1+x^2)} \ud x \nonumber \\
&= - \int_{0}^{b(\Btheta,\Bphi^i)} a(\Btheta,\Bphi^i) \frac{\partial a(\Btheta,\Bphi^i)}{\partial \Btheta} \e^{-\frac{1}{2}a^2(\Btheta,\Bphi^i)(1+x^2)} \ud x \nonumber \\
&= - a(\Btheta,\Bphi^i) \frac{\partial a(\Btheta,\Bphi^i)}{\partial \Btheta} \e^{-\frac{1}{2}a^2(\Btheta,\Bphi^i)} \int_{0}^{b(\Btheta,\Bphi^i)} \e^{-\frac{1}{2}a^2(\Btheta,\Bphi^i) x^2} \ud x \nonumber \\
&= - \sqrt{2\pi} \, \frac{\partial a(\Btheta,\Bphi^i)}{\partial \Btheta} \e^{-\frac{1}{2}a^2(\Btheta,\Bphi^i)} \left(\Phi(a(\Btheta,\Bphi^i)b(\Btheta,\Bphi^i)) - \Phi(0)\right) \nonumber \\ 
&= \sqrt{\pi/2} \, \e^{-\frac{1}{2}a^2(\Btheta,\Bphi^i)} \left(1 - 2\Phi(a(\Btheta,\Bphi^i)b(\Btheta,\Bphi^i))\right) \frac{\partial a(\Btheta,\Bphi^i)}{\partial \Btheta}, \label{eq:sec_term_final}
\end{align}
%
where on the third line we have recognised the integrand as an unnormalised Gaussian pdf. 
%
%
Finally, straightforward calculations show that 
\begin{align}
\frac{\partial a(\Btheta,\Bphi^i)}{\partial \Btheta} 
&= -\frac{1}{\sqrt{(\sigma^i_n)^2 + v_{\idxt}^2(\Btheta,\Bphi^i)}} \frac{\partial m_{\idxt}(\Btheta,\Bphi^i)}{\partial \Btheta} - \frac{\epsilon - m_{\idxt}(\Btheta,\Bphi^i)}{2((\sigma^i_n)^2 + v_{\idxt}^2(\Btheta,\Bphi^i))^{3/2}} \frac{\partial v_{\idxt}^2(\Btheta,\Bphi^i)}{\partial \Btheta}, \label{eq:dadtheta} \\ 
\frac{\partial b(\Btheta,\Bphi^i)}{\partial \Btheta} 
&= -\frac{\sigma_n^i}{((\sigma^i_n)^2 + 2v_{\idxt}^2(\Btheta,\Bphi^i))^{3/2}} \frac{\partial v_{\idxt}^2(\Btheta,\Bphi^i)}{\partial \Btheta}. \label{eq:dbdtheta}
\end{align}
Gradients of the GP mean and variance functions $m$ and $v^2$ depend on the chosen covariance function and are not shown here.

\subsection{Gradient of the ABC posterior approximation} \label{app:derivatives}

The gradient of the posterior approximation with respect to parameter $\Btheta$ is useful if e.g.~Hamiltonian Monte Carlo algorithm is used to sample from this density. The unnormalised approximate posterior is 
\begin{align}
\tildepiabc(\Btheta \cond x_{obs}) 
= \pi(\Btheta) \, \Phi(a(\Btheta))
%
\label{eq:post_approx_appendix} 
\end{align}
where, as earlier, $\pi(\Btheta)$ denotes the prior density and $a$ is defined as in Equation (\ref{eq:ab}). 
Differentiating the logarithm of Equation (\ref{eq:post_approx_appendix}) yields
\begin{equation}
\begin{aligned}
\frac{\partial }{\partial \Btheta} \log \tildepiabc(\Btheta \cond x_{obs}) 
&= \frac{\partial \log \pi(\Btheta)}{\partial \Btheta} + \frac{\partial }{\partial \Btheta} \log \Phi(a(\Btheta)) 
%
= \frac{\partial \log \pi(\Btheta)}{\partial \Btheta} + \frac{\e^{-\frac{1}{2}a^2(\Btheta)}}{\sqrt{2\pi} \, \Phi(a(\Btheta))} \frac{\partial a(\Btheta)}{\partial \Btheta},
\end{aligned}
\end{equation}
where the gradient of the term $a$ can be computed as in Equation (\ref{eq:dadtheta}). 

\section{Additional details and experiments} \label{app:exp_details}

{\color{\revcol} In this section we first briefly discuss the computational cost of our algorithm using Big O notation and then we provide some further details and results of our experiments. 

We consider the cost of evaluating the \expintvar{} acquisition function at $b$ arbitrary points $\Btheta^{*}_1,\ldots,\Btheta^{*}_b$ at the iteration $t$ of Algorithm~\ref{alg:gp_abc_alg}. Computing the Cholesky factorisation on line 8 of Algorithm~\ref{alg:gp_abc_alg} requires $\bigO(t^3)$. Generating $s$ samples from the proposal $\pi_q$ on line 9 requires computing GP mean and variance functions for each proposed point by the MCMC algorithm and the total cost is $\bigO(st^2)$ when the precomputed Cholesky is used. We can reuse the GP mean and variance function values and line 10 thus requires $\bigO(\tilde{s})$ where $\tilde{s}$ is the amount of thinned samples that are used for approximating the integral in Equation~(\ref{eq:expintvar}). We obviously have $\tilde{s} \leq s$. (In fact, the second term of Equation~(\ref{eq:expintvar}) is constant with respect to $\Btheta^*$ and we may not need to evaluate it. However, from the previous discussion we can see that the resulting saving would be small i.e.~only $\bigO(\tilde{s})$.) Computing the value of Equation (\ref{eq:dgpvar}) at the $\tilde{s}$ sampled locations and for all $b$ values of $\Btheta^*$ can be seen to be $\bigO(b\tilde{s}t + bt^2)$ when we use the already computed values of $k(\Btheta,\Btheta_{\idxt})K^{-1}(\Btheta_{\idxt})$ in the formula of $\cov_{1:t}(\Btheta,\Btheta^*)$. Computing the integral over the first term then requires $\bigO(\tilde{s})$. 
The total cost is thus $\bigO(t^3 + st^2 + b(t^2 + \tilde{s}t))$. 
If we use grid approximation instead of importance sampling to approximate the integral in Equation~(\ref{eq:expintvar}), we obtain the total cost by setting $s=\tilde{s}$ where $s$ is now the number of grid points. 

We obtain $\bigO(t^3 + bt^2)$ bound for the cost of LCB, \maxvar{} and \expdiffvar{} and we see that the $st^2 + b\tilde{s}t$ term is missing as compared to the corresponding bound of the \expintvar{} acquisition function. There is no acquisition function to evaluate (or optimise) for \randmaxvar{} rule since we choose the next point directly sampling from a density. This can be done in $\bigO(t^3 + st^2)$. Finally, we want to emphasise that this analysis is asymptotic and in practice the constant costs of \expintvar{} and its variants due to the need to compute e.g.~the Owens t-function and cdf of the standard Gaussian are slightly higher as compared to e.g.~LCB rule. However, in practice all these GP computation times are negligible when the simulation time dominates the computation.
}

{\color{\revcol} In the rest of this section we show further details and results of the experiments.}
The synthetic test problems in Section \ref{sec:synth_experiments} are designed in the following way. 
In the ``unimodal'' example, the mean of the discrepancy is $m(\Btheta) = 3\sigma + \Btheta^T \BS \Btheta$, where $\sigma$ is the standard deviation of the additive Gaussian noise, $\BS_{11} = \BS_{22} = 1$ and $\BS_{12} = \BS_{21} = 0.5$. 
In the ``bimodal'' example we use 
$m(\Btheta) = 3\sigma + 0.2(\theta_2-\theta_1^2)^2 + 0.75(\theta_2-\theta_1-2)^2$.
In the unidentifiable example, the mean of the discrepancy is obtained as $m(\Btheta) = 3\sigma + 0.01\theta_1^2 + \theta_2^2$. 
The ``banana'' example is produced using $m(\Btheta) = 3\sigma + (1-\theta_1)^2 + 10(\theta_2 - \theta_1^2)^2$. The discrepancy is assumed to follow the Gaussian density, that is, $\Delta_{\Btheta} \sim \Normal(m(\Btheta),\sigma^2)$. 
The resulting probability densities with $\sigma = 2$ are also illustrated in Figure \ref{fig:synth_problems}. 

We present additional results for the 2d experiments in Section \ref{sec:synth_experiments}. The settings are the same except that the threshold is not predefined but is set to the $0.01$th quantile of the realised discrepancies, and updated constantly during the acquisitions. Consequently, the selection of the evaluation locations also affects how the threshold is chosen. These results are shown in Figure \ref{fig:synth_eps_q001}. 
Figure \ref{fig:synth_eps_q005} shows the corresponding results when the threshold is determined using the 0.05th quantile. 

\begin{figure*}[ht]
\centering
\includegraphics[width=0.65\textwidth]{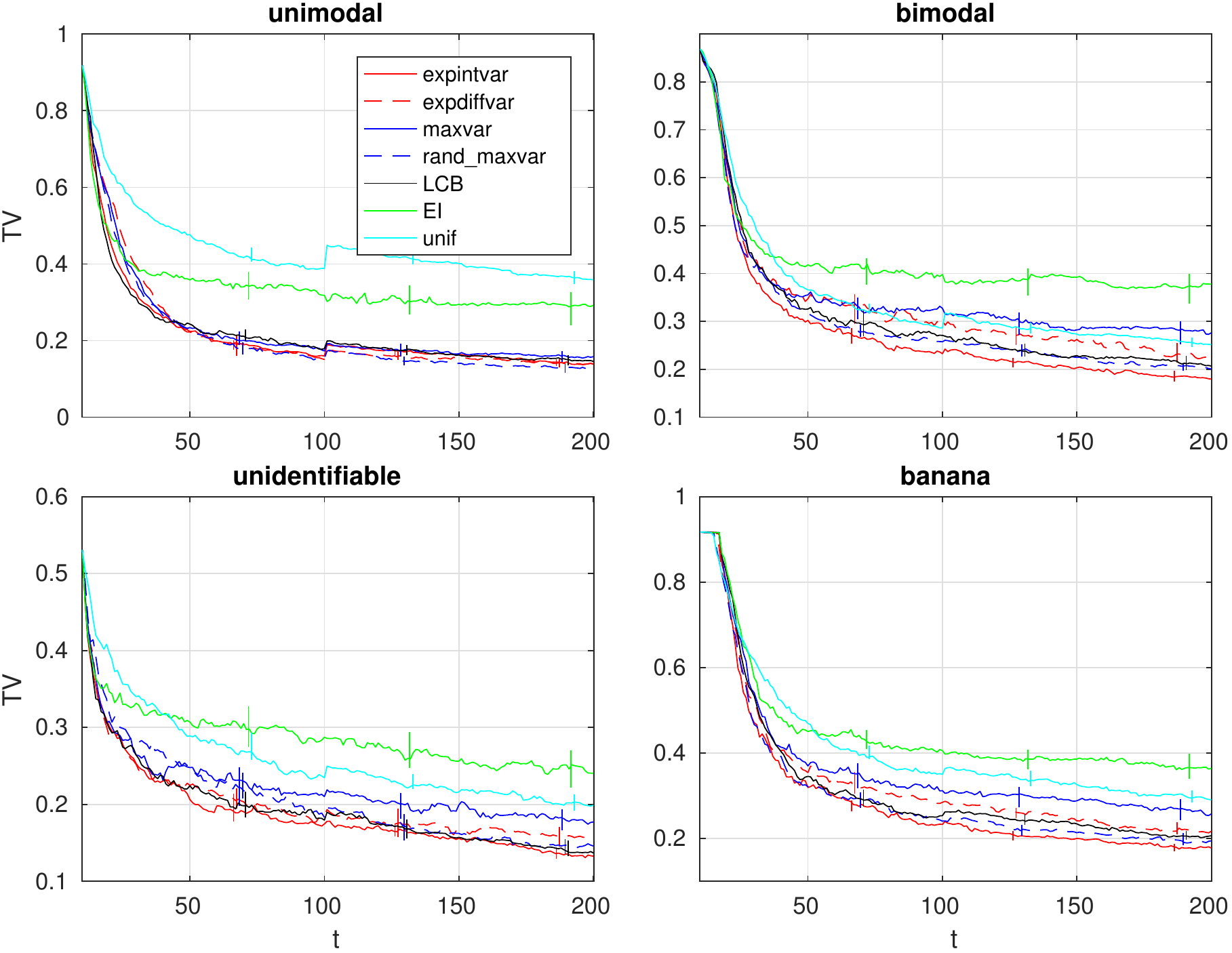}
\caption{Median of the TV distance between the estimated and the true {\color{\revcol} ABC} posterior over 100 experiments. The 0.01th quantile is used for updating the threshold. The results are similar as in Figure \ref{fig:synth_eps_fixed}.}
\label{fig:synth_eps_q001}
\end{figure*}

\begin{figure*}[ht]
\centering
\includegraphics[width=0.65\textwidth]{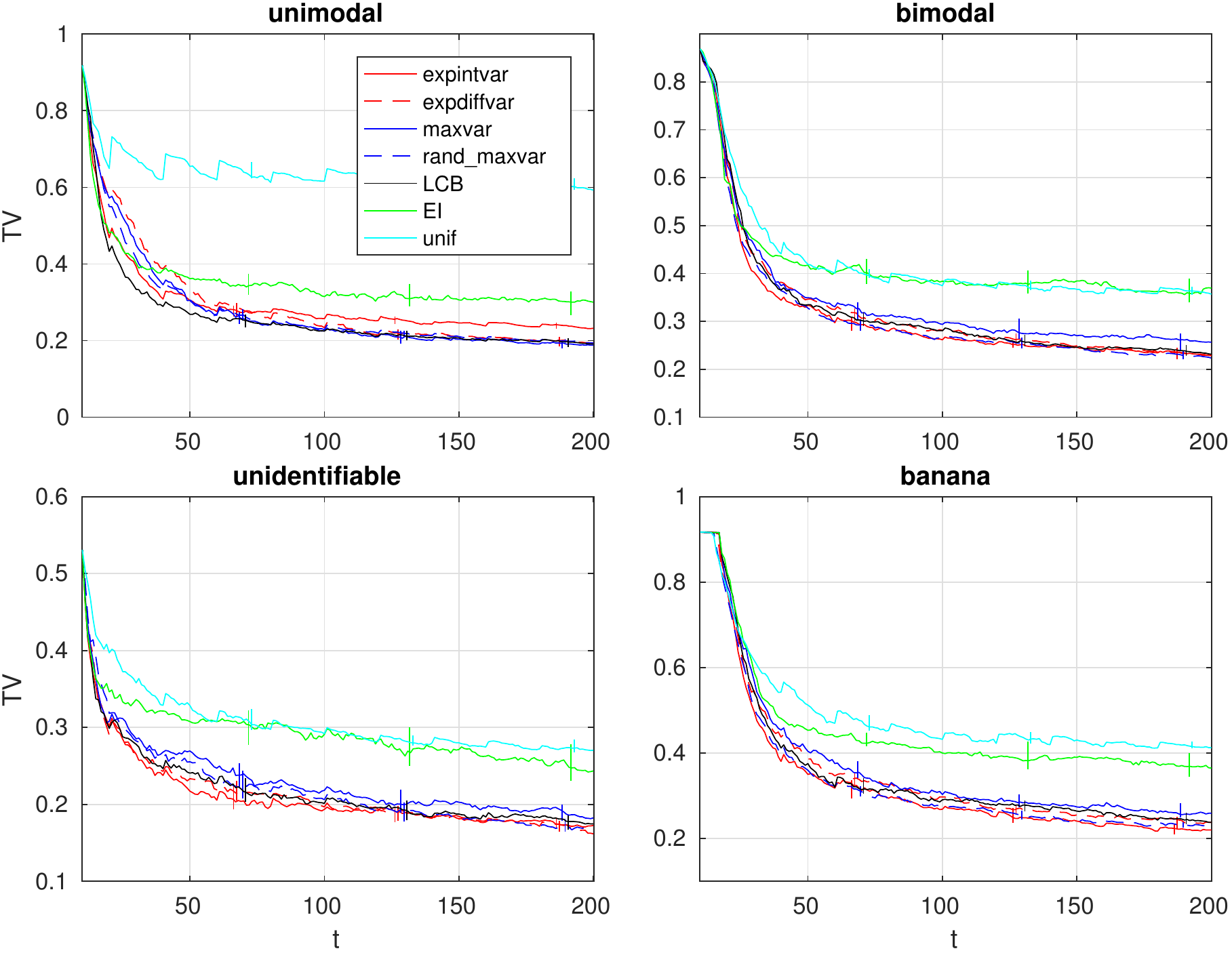}
\caption{Median of the TV distance between the estimated and the true {\color{\revcol} ABC} posterior over 100 experiments. These experiments are as in Figure \ref{fig:synth_eps_q001} except that the 0.05th quantile is used. Larger threshold than in Figure \ref{fig:synth_eps_q001} generally produces slightly worse posterior estimates.}
\label{fig:synth_eps_q005}
\end{figure*}

\section{Non-uniform acceptance threshold} \label{eq:gaussian_kernel}

Instead of using the uniform (i.e.~``0-1'') threshold
$\pi_{\epsilon}(\Bx_{\text{obs}}\cond\Bx) \propto \indic_{\Delta(\Bx_{obs},\Bx) \leq \epsilon}$, 
other choices are possible. For instance, one can use ``Gaussian threshold'' 
$\pi_{\epsilon}(\Bx_{\text{obs}}\cond\Bx) \propto \Normal({\Delta(\Bx_{\text{obs}},\Bx) \cond m_{\epsilon},\sigma_{\epsilon}^2})$,
where the threshold $\epsilon$ is replaced by two new parameters $m_\epsilon$ and $\sigma_{\epsilon}^2$ that control the quality of the ABC approximation. 
%
The unnormalised ABC posterior approximation at $\Btheta$ is then given by 
\begin{align}
\tildepiNabc(\Btheta) 
&= \pi(\Btheta) \int_{-\infty}^{\infty} \Normal(\Delta \cond m_{\epsilon}, \sigma_{\epsilon}^2) \Normal(\Delta \cond f(\Btheta), \sigma_n^2) \ud \Delta \\
&= \pi(\Btheta) \Normal(f(\Btheta) \cond m_{\epsilon},\sigma^2_{\epsilon} + \sigma^2_n).
\end{align}
This approach can be seen as an approximation to the uniform threshold but it could be interpreted also as additional Gaussian measurement (or modelling) error as described by \citet{Wilkinson2013}.

Proceeding similarly as in the proof of Lemma \ref{lemma:p_mean_var} and using Gaussian identities in the appendix of \citet{Rasmussen2006} to compute the required integrals, the expectation and variance of $\tildepiNabc$ can be shown to be 
\begin{align}
\mean(\tildepiNabc(\Btheta) \cond D_{1:t}) 
&= \pi(\Btheta) \Normal(m_{\epsilon} \cond m_{\idxt}(\Btheta), \sigma^2 + v_{\idxt}^2(\Btheta)), \\
\Var(\tildepiNabc(\Btheta) \cond D_{1:t}) 
&= \frac{\pi^2(\Btheta)}{2\sqrt{\pi \sigma^2}} \Normal\left(m_\epsilon \mcond m_{\idxt}(\Btheta), \frac{\sigma^2}{2} + v_{\idxt}^2(\Btheta)\right) 
- \pi^2(\Btheta) \left[\Normal(m_\epsilon \cond m_{\idxt}(\Btheta), \sigma^2 + v_{\idxt}^2(\Btheta))\right]^2,
\end{align}
where $\sigma^2 = \sigma^2_{\epsilon} + \sigma^2_n$. 
%
In addition, the expected integrated variance acquisition function can now be written 
\begin{align}
L_{\idxt}(\Btheta^{*})  
%
&= \int_{\Theta} \pi^2(\Btheta) \left[ \frac{\Normal\left(m_\epsilon \mcond m_{\idxt}(\Btheta), \frac{\sigma^2}{2} + v_{\idxt}^2(\Btheta)\right)}{2\sqrt{\pi \sigma^2}} 
%
- \frac{\Normal\left(m_\epsilon \mcond m_{\idxt}(\Btheta), \frac{\sigma^2 + v_{\idxt}^2(\Btheta) - \Deltav_{\idxt}^2(\Btheta,\Btheta^{*})}{2}\right)}{2\sqrt{\pi}\sqrt{\sigma^2 + v_{\idxt}^2(\Btheta) + \Deltav_{\idxt}^2(\Btheta,\Btheta^{*})}} \right] \ud \Btheta, 
\label{eq:expintvar_gaussian}
\end{align}
where the computations follow similarly as in the proof of Proposition \ref{prop:expintvar} and by applying Gaussian identities to compute the integrals. (Details are left to the reader.)

The advantage of this approach is that one avoids Owen's t-function evaluations. However, we found determining the two threshold values $m_{\epsilon}$ and $\sigma_{\epsilon}^2$ more challenging than setting the threshold $\epsilon$ for the uniform threshold and thus focused on the latter approach. 
On the other hand, while running the simulation model typically dominates the total computational cost, these formulae may be useful in high-dimensions where the global optimisation of the acquisition function can also be costly.

\label{lastpage}
\end{document}